\documentclass[final,hidelinks,onefignum,onetabnum]{siamart251216}

\usepackage{lipsum}
\usepackage{amsfonts}
\usepackage{graphicx}
\usepackage{epstopdf}
\usepackage{algorithmic}
\ifpdf
  \DeclareGraphicsExtensions{.eps,.pdf,.png,.jpg}
\else
  \DeclareGraphicsExtensions{.eps}
\fi

\newsiamremark{remark}{Remark}
\newsiamremark{hypothesis}{Hypothesis}
\crefname{hypothesis}{Hypothesis}{Hypotheses}
\newsiamthm{claim}{Claim}
\newsiamremark{fact}{Fact}
\crefname{fact}{Fact}{Facts}

\headers{LR2Flow}{
Chenglong Bao, Tongyao Pang, Zuowei Shen, Dihan Zheng, Yihang Zou}

\title{Enhancing Low-resolution Image Representation Through Normalizing Flows\thanks{Submitted to the editors DATE.
\funding{This work was partially supported by the National Natural Science Foundation of China (No. 12625119, W2612010) and
the New Cornerstone Investigator Program (NCI202310).}}}

\author{
Chenglong Bao\thanks{Yau Mathematical Sciences Center, Tsinghua University, Beijing, China, and Beijing Institute of Mathematical Sciences and Applications, Beijing, China(
\email{clbao@tsinghua.edu.cn})}
\and
Tongyao Pang\thanks{Yau Mathematical Sciences Center, Tsinghua University, Beijing, China(\email{typang@tsinghua.edu.cn}).}
\and
Zuowei Shen\thanks{Department of Mathematics, National University of Singapore, Singapore (\email{matzuows@nus.edu.sg}).}
\and 
Dihan Zheng\thanks{Department of Pharmaceutical Chemistry, University of California, San Francisco, CA, USA (\email{dihan.zheng@ucsf.edu}).}
\and
Yihang Zou\thanks{Yau Mathematical Sciences Center and Department of Mathematical Sciences, Tsinghua University, Beijing, China} (\email{zou-yh21@mails.tsinghua.edu.cn}).
}

\usepackage{amsopn}

\usepackage{graphicx}
\usepackage{multirow}
\usepackage{makecell} 
\usepackage{pifont}

\newcommand{\cmark}{\ding{51}}%
\newcommand{\xmark}{\ding{55}}%

\ifpdf
\hypersetup{
pdftitle={Enhancing Low-resolution Image Representation Through Normalizing Flows},
pdfauthor={
Chenglong Bao, Tongyao Pang, Zuowei Shen, Dihan Zheng, Yihang Zou}
}
\fi

\usepackage{amsmath,amsfonts,bm}

\newcommand{\df}[1]{\mathrm{d}#1}

\def\rva{{\mathbf{a}}}
\def\rvb{{\mathbf{b}}}
\def\rvc{{\mathbf{c}}}

\def\rve{{\mathbf{e}}}

\def\rvu{{\mathbf{i}}}

\def\rvk{{\mathbf{k}}}

\def\rvn{{\mathbf{n}}}

\def\rvs{{\mathbf{s}}}

\def\rvu{{\mathbf{u}}}

\def\rvx{{\mathbf{x}}}
\def\rvy{{\mathbf{y}}}
\def\rvz{{\mathbf{z}}}

\def\vone{{\bm{1}}}

\def\mSigma{{\bm{\Sigma}}}

\def\mA{{\bm{A}}}
\def\mB{{\bm{B}}}
\def\mC{{\bm{C}}}

\def\mF{{\bm{F}}}

\def\mI{{\bm{I}}}

\def\mK{{\bm{K}}}

\def\mP{{\bm{P}}}
\def\mQ{{\bm{Q}}}

\def\mS{{\bm{S}}}

\def\mU{{\bm{U}}}
\def\mV{{\bm{V}}}
\def\mW{{\bm{W}}}
\def\mX{{\bm{X}}}
\def\mY{{\bm{Y}}}

\def\mSigma{{\bm{\Sigma}}}

\DeclareMathAlphabet{\mathsfit}{\encodingdefault}{\sfdefault}{m}{sl}
\SetMathAlphabet{\mathsfit}{bold}{\encodingdefault}{\sfdefault}{bx}{n}

\def\gC{{\mathcal{C}}}

\def\gE{{\mathcal{E}}}
\def\gF{{\mathcal{F}}}

\def\gI{{\mathcal{I}}}

\def\gK{{\mathcal{K}}}
\def\gL{{\mathcal{L}}}

\def\gN{{\mathcal{N}}}

\def\gR{{\mathcal{R}}}

\def\gT{{\mathcal{T}}}

\def\sI{{\mathbb{I}}}

\def\sR{{\mathbb{R}}}
\def\sS{{\mathbb{S}}}

\newcommand{\E}{\mathbb{E}}

\newcommand{\Var}{\mathrm{Var}}

\newcommand{\Cov}{\mathrm{Cov}}
\newcommand{\Tr}{\mathrm{Tr}}

\usepackage[skip=0.333\baselineskip]{caption}
\usepackage{tikz}
\usepackage{makecell}
\usepackage{mathtools}
\usepackage{multirow, multicol}
\usepackage{stdclsdv}   

\usepackage[dvipsnames]{xcolor}
\usepackage[markup=underlined]{changes}

\newtheorem*{example}{Example}

\begin{document}

\maketitle

\begin{abstract} 
Low-resolution image representation can be regarded as a special form of sparse representation that retains only low-frequency information while discarding high-frequency components. This property reduces storage and transmission costs and benefits various image processing tasks. However, a key challenge is to preserve essential visual content while maintaining the ability to accurately reconstruct the original images. This work proposes LR2Flow, a nonlinear framework that learns low-resolution image representations by integrating wavelet tight frame blocks with normalizing flows. We conduct a reconstruction error analysis of the proposed network, which demonstrates the necessity of designing invertible neural networks in the wavelet tight frame domain. Experimental results on various tasks, including image rescaling, compression, and denoising, demonstrate the effectiveness of the learned representations and the robustness of the proposed framework. Code is available at \href{https://github.com/Evanescentlove/LR2Flow-main}{https://github.com/Evanescentlove/LR2Flow-main}.
\end{abstract}



\section{Introduction}
With the rapid advancement of imaging devices, both the volume and resolution of acquired images have increased significantly, posing substantial challenges for transmission and storage. While applying standard compression techniques or reducing resolution via downsampling are natural approaches, they raise a fundamental question: \emph{How can one obtain a low-resolution (LR) image that allows for the accurate reconstruction of its high-resolution (HR) counterpart?}
This challenge is known as the \emph{LR representation} problem.
Mathematically, let $\rvx \in \mathbb{R}^n$ denote an HR image, and let its LR representation reside in $\mathbb{R}^d$ with $d < n$. Our objective is to design a downscaling operator $\varphi:\mathbb{R}^{n}\to\mathbb{R}^{d}$ and an upscaling operator $\psi:\mathbb{R}^{d}\to\mathbb{R}^{n}$ that satisfy
\begin{equation}
\label{eqn:goal}
\rvx \approx \psi(\varphi(\rvx)) 
\quad \text{and} \quad 
\varphi(\rvx) \approx \downarrow_{n/d}(\rvk^{\text{lp}} \ast \rvx),
\end{equation}
where $\rvk^{\text{lp}}$, $\ast$, and $\downarrow_{n/d}$ denote a predefined low-pass filter, the convolution operator, and the downsampling operator with a rate of $n/d$, respectively. We refer to $\varphi(\rvx)$ as the LR representation of $\rvx$.
The first condition in~\eqref{eqn:goal} ensures the reconstructibility of the HR image, while the second encourages the preservation of global appearance. Consequently, this model yields a compact surrogate representation well-suited for compression and edge-to-cloud collaborative inference in low-bandwidth scenarios~\cite{choi2022scalable, xiao2020invertible, shao2020bottlenet++}.
The resulting framework is applicable to a wide range of tasks, including image rescaling~\cite{xiao2020invertible,liang2021hierarchical,bao2025tinvblock,wang2025timestep}, image compression~\cite{bao2025tinvblock,yang2023self,xiao2023invertible}, and image restoration~\cite{liu2021invertible}.

LR representation is closely related to sparse representation. In the classical setting, one seeks a sparse coefficient vector $\rvc \in \mathbb{R}^{N}$ for an image $\rvx \in \mathbb{R}^{n}$ such that $\mW^{\top}\rvc \approx \rvx$, where $\mW$ is a predefined wavelet tight frame~\cite{ron1997affine,daubechies2003framelets} or a data-driven transformation~\cite{aharon2006k,quan2015data,cai2014data}. One typical construction of $\mW$ involves selecting a low-pass filter $\rvk^{\text{lp}}$ and a set of high-pass filters $\{\rvk^{\text{hp}}_{i}\}_{i=1}^{r}$ that satisfy certain extension principles. We define $\mW = (\mW_L^{\top}, \mW_H^{\top})^{\top}$, where the low- and high-frequency decomposition operators are given by
\begin{equation*}
\mW_L: \rvx \mapsto \downarrow_{2}(\rvk^{\text{lp}} \ast \rvx)\in\sR^{\frac{n}{2}},
\quad
\mW_H: \rvx \mapsto \big(\downarrow_{2}(\rvk^{\text{hp}}_{1} \ast \rvx); \ldots; \downarrow_{2}(\rvk^{\text{hp}}_{r} \ast \rvx) \big)\in\sR^{\frac{nr}{2}}
\end{equation*}
The sparse coefficient $\rvc$ is typically formed by preserving the low-frequency coefficient $\rvx_{L} = \mW_L\rvx$ and applying a nonlinear thresholding operator (or a variant)~\cite{donoho2002noising} to the high-frequency coefficients $\rvx_{H}=\mW_{H}\rvx$. This results in the representation $\rvc = (\rvx_{L}, \gT(\rvx_{H}))$, where $\gT$ denotes the thresholding-induced nonlinear mapping. Thus, $\rvc$ becomes an LR representation when $\gT(\rvx_{H}) = \mathbf{0}$. In this setting, $d=n/2$, the downscaling and upscaling operators are 
\begin{equation*}
\varphi: \rvx \mapsto [\mW\rvx]_{1:d},
\quad
\psi:\rvy \mapsto \mW^{\top}(\rvy; \mathbf{0}_{N-d}),
\end{equation*}
where $\rvu_{s:t} = (u_{i})_{s\leq i \leq t}$ denotes sub-vectors, and $\mathbf{0}_{N-d} \in \sR^{N-d}$ is a zero vector. While the downscaling operator \(\varphi(\rvx)\) preserves the primary visual features of \(\rvx\), the reconstructed image \(\psi(\varphi(\rvx)) = \mW_L^{\top}\mW_L \rvx\) inevitably suffers from a loss of high-frequency information.

Motivated by the above observations, we propose a framework that jointly learns the downscaling and upscaling operators within a designed \emph{invertible} architecture, with the goal of enhancing the recovery of high-frequency information. 
Specifically, let $\mW \in \mathbb{R}^{N \times n}$ be a wavelet tight frame and let $\gF:\mathbb{R}^N \to \mathbb{R}^N$ denote an invertible mapping. 
The downscaling and upscaling operators $\varphi$ and $\psi$ are defined as
\begin{equation*}
\varphi:\ \rvx \mapsto \bigl[\gF(\mW \rvx)\bigr]_{1:d},
\qquad 
\psi:\ \rvy \mapsto \mW^{\top}\gF^{-1} \circ \gE(\rvy),
\end{equation*}
where $\gE:\mathbb{R}^{d} \to \mathbb{R}^{N}$ is an extension operator. 
By construction, this design yields a compact latent representation.
To further enhance reconstruction, we map the remaining coordinates $[\gF(\mW\rvx)]_{(d+1):N}$ to a prior distribution $p$. This allows us to discard these components during transmission and restore them via sampling, effectively defining $\gE(\rvy) = (\rvy, \rvz)$ where $\rvz \sim p$. Consequently, the final downscaling and upscaling operators are formulated as
\begin{equation}
\label{eq:flow_model}
\varphi(\rvx) = \left[\gF\left(\mW \rvx\right)\right]_{1:d},
\quad 
\psi(\rvy) = \E_{p(\rvz)}\left[\mW^\top\gF^{-1}\left(\rvy,\rvz\right)\right].
\end{equation}

We parameterize $\gF$ using invertible neural networks, such as Normalizing Flows (NFs)~\cite{dinh2016density,kingma2018glow} or Invertible Residual Networks (iResNets)~\cite{behrmann2019invertible}, and apply this architecture progressively to the wavelet tight frame coefficients at each level. From a theoretical perspective, we analyze the reconstruction error of $\hat{\rvx} = \psi(\varphi(\rvx))$ obtained from the learned mapping, which provides insight into the necessity of designing a nonlinear $\gF$ to align with the underlying data distribution in order to obtain an effective LR representation. Furthermore, our derived reconstruction error bound demonstrates the theoretical advantage of using a redundant wavelet tight frame system over an orthonormal basis, which is also verified empirically through numerical experiments. Finally, we evaluate our model across various image processing tasks, demonstrating its superior performance and broad applicability.

\textbf{Notations.} We denote vectors using bold lowercase letters (e.g., $\rva$), where $a_i$ refers to the $i$-th component. Matrices are denoted by bold uppercase letters (e.g., $\mA$), with $A_{ij}$ representing the entry in the $i$-th row and $j$-th column. For a vector $\rvu \in \mathbb{R}^{m}$ and integers $1\leq s\leq t\leq m$, we define the slice $\rvu_{s:t} = (u_{i})_{s\leq i \leq t}$. The $\ell_p$-norm of a vector $\rvu$ is defined as $\|\rvu\|_{p} = \left(\sum_{i}|u_i|^p\right)^{1/p}$. Unless otherwise specified, $\|\rvu\|$ denotes the Euclidean norm $\|\rvu\|_{2}$.

\section{Related Work}

In this section, we briefly review related work on LR representation in the context of image rescaling and compression.

Image rescaling seeks a forward mapping from HR domain to LR representation and an inverse mapping for reconstruction. The goal is to produce visually coherent LR images that allow for faithful HR restoration. Prior works generally fall into three categories: (i) methods that fix the downscaling kernel (typically Bicubic interpolation) and train a super-resolution (SR) network for upscaling~\cite{dong2015image,ahn2018fast,lim2017enhanced,zhang2018image,dai2019second,liang2021swinir,chen2023activating}; (ii) frameworks that jointly optimize downscaling and upscaling in an end-to-end encoder-decoder architecture~\cite{wang2025timestep,kim2018task,sun2020learned}; and (iii) approaches employing NFs to model these stages as mutually invertible processes~\cite{xiao2020invertible,liang2021hierarchical,bao2025tinvblock,xiao2023invertible}. However, categories (i) and (ii) often exhibit limited synergy between stages. By coupling downscaling and upscaling primarily through reconstruction loss, these methods lack geometric regularization and fail to exploit the reciprocal structure of the task or account for information loss inherent in downscaling~\cite{xiao2020invertible}. In contrast, flow-based methods formulate rescaling as a multiscale bijection, enabling both visually consistent LR representations and robust reconstruction.

Although image rescaling inherently compresses data by reducing spatial resolution, applying standard lossy compression (e.g., JPEG) to the LR representation allows for significantly higher compression ratios. Consequently, this approach has emerged as a prominent research direction. Recent works~\cite{xing2023scale,yang2023self,bao2025tinvblock} implement this strategy by employing a differentiable JPEG simulator to facilitate end-to-end training of the rescaling-based compression network. Furthermore, to mitigate distribution shifts induced by lossy compression in the LR domain, \cite{yang2023self,bao2025tinvblock,xiao2023invertible} utilize auxiliary JPEG decoders to eliminate compression artifacts.

While flow-based architectures have achieved significant success in image rescaling and compression, existing designs predominantly rely on a fixed orthogonal basis for frequency decomposition~\cite{liu2021invertible,xiao2020invertible,xiao2023invertible,liang2021hierarchical,bao2025tinvblock,yang2023self}. Despite their empirical efficacy, a rigorous theoretical analysis of the reconstruction error in these methods remains largely unexplored.
Furthermore, it is well-established that redundant representations enhance stability in image reconstruction~\cite{elad2006image,coifman1995translation,cai2009linearized,goyal2001quantized}, as the corresponding inverse transformation operator possesses a nontrivial kernel. Motivated by this, we integrate invertible neural networks with wavelet tight frames to capitalize on these advantages. To the best of our knowledge, we present the first theoretical reconstruction error analysis for such an architecture.

\section{The LR2Flow Model}

This section presents the architecture of the proposed LR representation model, 
referred to as LR2Flow. Subsequently, we provide the approximation analysis of this model.

\subsection{The Construction of LR2Flow}
\label{subsec:model_arch}
Recall from~\eqref{eq:flow_model} that the downscaling and upscaling operators are 
\(\varphi(\rvx) = \bigl[\gF(\mW \rvx)\bigr]_{1:d}\) and
\(\psi(\rvy) = \mathbb{E}_{p(\rvz)}\!\left[\mW^{\top}\gF^{-1}(\rvy,\rvz)\right]\).
In this work, we assume that \(\rvz\) is a random variable drawn from a standard Gaussian distribution. The operator \(\mW = (\mW_L^{\top}, \mW_H^{\top})^{\top}\) is chosen as a wavelet tight frame generated by linear B-splines~\cite{cai2012image}, which consists of one low-pass filter and two high-pass filters. Consequently, this transformation decomposes an input image \(\rvx \in \mathbb{R}^{n}\) into a low-frequency coefficient vector \(\rvx_L = \mW_L \rvx \in \mathbb{R}^{n/2}\) and a high-frequency coefficient vector \(\rvx_H = \mW_H \rvx \in \mathbb{R}^{n}\).
Assuming the LR representation dimension is $d=n/2^T$ for some integer $T$, we represent $\gF$ as the composition $\gF = \gF^{(T)}\circ\gF^{(T-1)}\circ\cdots\circ\gF^{(1)}$, where
\begin{equation*}
\gF^{(l)}=
\begin{cases}
f^{(1)}: \sR^{\frac{3n}{2}} \to \sR^{\frac{3n}{2}}, & l=1; \\
(f^{(l)} \circ \mW) \times \text{Id}_{h_{l-1}}: \sR^{\frac{n}{2^{l-1}}}\times \sR^{h_{l-1}} \to \sR^{\frac{3n}{2^{l-1}}}\times \sR^{h_{l-1}}, & 2\leq l\leq T.
\end{cases}
\end{equation*}
Here, $h_l = 2n(1-2^{-l})$ denotes the total dimension of the high-frequency components accumulated from the previous $l$ levels, $\mathrm{Id}_{h_l}:\sR^{h_l}\to \sR^{h_l}$ represents the identity mapping, and $f^{(l)}:\sR^{\frac{3n}{2^{l}}}\to \sR^{\frac{3n}{2^{l}}}$ constitutes the flow at the $l$-th level. Specifically, $f^{(l)}$ comprises a sequence of $M$ flow blocks, denoted as $f^{(l)} = f^{(l)}_{M} \circ \cdots \circ f^{(l)}_{1}$. Each block $f^{(l)}_{i}$ consists of an ActNorm layer, an invertible $1\times 1$ convolution~\cite{kingma2018glow}, and a nonlinear invertible transformation $g^{(l)}_{i}$:
\begin{equation*}
f^{(l)}_{i} = g^{(l)}_{i} \circ \text{Inv}_{1\times1} \circ \text{ActNorm}.
\end{equation*}
To define these components, let the input be $\rvc^{(l)} = (\rvc^{(l)}_{0}, \ldots, \rvc^{(l)}_{r})\in\sR^{m_{l}}$, where $m_l = \frac{3n}{2^l}$ is the total dimension of the wavelet coefficients at level $l$. Recall that $r=2$ represents the number of high-pass filters; each sub-vector $\rvc^{(l)}_{k}$ lies in $\sR^{\frac{m_{l}}{r+1}}$. Here, $\rvc^{(l)}_{0}$ corresponds to the low-frequency component, while $\rvc^{(l)}_{k}$ corresponds to the $k$-th high-frequency subband for $1\leq k\leq r$.
The ActNorm transformation, parameterized by a scale vector $\rvs=(s_{0}, \ldots, s_{r})\in\sR_{>0}^{r+1}$ and a bias vector $\rvb =(b_{0}, \ldots, b_{r})\in\sR^{r+1}$, is defined as
\begin{equation}
\label{eq:actnorm}
\text{ActNorm}\big(\rvc^{(l)}; \rvs, \rvb) = \big(
s_{0}\mathbf{1}_{\frac{m_l}{r+1}} \odot \rvc^{(l)}_{0} + b_{0}\mathbf{1}_{\frac{m_l}{r+1}},
\ldots, 
s_{r}\mathbf{1}_{\frac{m_l}{r+1}} \odot \rvc^{(l)}_{r} + b_{r}\mathbf{1}_{\frac{m_l}{r+1}}
\big),
\end{equation}
where $\mathbf{1}_{{m_l}/{(r+1)}}\in\sR^{\frac{m_l}{r+1}}$ is a vector of ones and $\odot$ denotes the Hadamard product. The invertible $1\times 1$ convolution, parameterized by an orthogonal matrix $\mK\in\sR^{(r+1)\times (r+1)}$, is expressed as
\begin{equation}
\label{eq:inv_1x1}
\text{Inv}_{1\times1}(\rvc^{(l)}; \mK) = (\mK\otimes\mI_{{m_l}/{(r+1)}})\rvc^{(l)},
\end{equation}
where $\mI_{{m_l}/{(r+1)}}\in\sR^{\frac{m_l}{r+1}\times \frac{m_l}{r+1}}$ denotes the identity matrix and $\otimes$ denotes the Kronecker product. Finally, the nonlinear invertible transformation $g^{(l)}_{i}$ typically adopts one of two designs:

{\bf (i) An affine coupling layer~\cite{dinh2016density}.}
This layer partitions the input $\rvc^{(l)}$ into a low-frequency part $\rvc^{(l)}_{L} = \rvc^{(l)}_{0}$ and a high-frequency part $\rvc^{(l)}_{H} = (\rvc^{(l)}_{1}, \ldots, \rvc^{(l)}_{r})$, defined as
\begin{equation}
\label{eq:affine_coupling}
g^{(l)}_{i}\big(\rvc^{(l)}_{L}, \rvc^{(l)}_{H}\big) = 
\big(\rvc^{(l)}_{L}, 
\rvc^{(l)}_{H} \odot \exp\big(\rho^{(l)}_{i}\big(\rvc^{(l)}_{L}\big)\big) + \eta^{(l)}_{i}\big(\rvc^{(l)}_{L}\big)\big),
\end{equation}
where $\rho^{(l)}_{i}, \eta^{(l)}_{i}: \sR^{\frac{m_l}{r+1}} \to \sR^{(1-\frac{1}{r+1})m_l}$ are parameterized neural networks.

{\bf (ii) An invertible residual block (iResBlock)~\cite{behrmann2019invertible}.} This block is defined as
\begin{equation}
\label{eq:iresblock}
g^{(l)}_{i}(\rvc^{(l)}) = \rvc^{(l)} + \phi^{(l)}_{i}(\rvc^{(l)}), 
\quad
\phi^{(l)}_{i}:\sR^{m_l} \to \sR^{m_l} \text {, }
\operatorname{Lip}(\phi^{(l)}_{i}) \leq L < 1.
\end{equation}


Since ActNorm, Inv$_{1\times1}$, and $g^{(l)}_{i}$ are invertible, each flow $f^{(l)}$ is guaranteed to be invertible. Combined with the perfect reconstruction property of $\mW$, i.e., $\mW^{\top}\mW=\mI$, the constructed $\gF$ constitutes a strictly invertible transformation with the inverse
$\gF^{-1} = (\gF^{(1)})^{-1}\circ \cdots \circ (\gF^{(T)})^{-1}$, where each $(\gF^{(l)})^{-1}$ is given by
\begin{equation}
\label{eq:inv_inn}
(\gF^{(l)})^{-1} = \begin{cases}(f^{(1)})^{-1}: \sR^{\frac{3n}{2}} \to \sR^{\frac{3n}{2}}, & l=1; \\
(\mW^{\top} \circ (f^{(l)})^{-1}) \times \text{Id}_{h_{l-1}}: \sR^{\frac{3n}{2^{l-1}}}\times \sR^{h_{l-1}} \to \sR^{\frac{n}{2^{l-1}}}\times \sR^{h_{l-1}}, & l\geq2.\end{cases}
\end{equation}
We now detail the computation of the downscaling and upscaling operators within the LR2Flow framework. The downscaling process proceeds hierarchically. Let $\rvx^{(0)} = \rvx \in \sR^{n}$ denote the input image. During the forward pass at the $l$-th level, the input $\rvx^{(l-1)}\in\sR^{\frac{n}{2^{l-1}}}$ is first decomposed via the wavelet tight frame into coefficients $(\rvx^{(l-1)}_L, \rvx^{(l-1)}_H) = \mW\rvx^{(l-1)}\in \sR^{\frac{n}{2^{l}}}\times \sR^{\frac{2n}{2^{l}}}$. Subsequently, the flow network computes $(\rvx^{(l)}, \rvz^{(l)}) = f^{(l)}(\rvx^{(l-1)}_L, \rvx^{(l-1)}_H)$. This operation fuses informative high-frequency details into the low-frequency subband to yield a refined representation $\rvx^{(l)}\in\sR^{\frac{n}{2^l}}$, while simultaneously encoding residual information into the latent variable $\rvz^{(l)}\in\sR^{\frac{n}{2^{l-1}}}$. Consequently, the forward transformation at the $l$-th level is expressed as
\begin{equation}
\label{eq:downscaling_forward}
(\rvx^{(l)}, \rvz^{(l)}) = f^{(l)}(\mW \rvx^{(l-1)}).
\end{equation}
Based on this notation, the downscaling operator for $\rvx$ is written as
\begin{equation*}
\varphi(\rvx) = \rvx^{(T)},
\quad
\text{where $\rvx^{(0)} = \rvx$, $\rvx^{(l)} = \big[f^{(l)}(\mW\rvx^{(l-1)})\big]_{1:n/2^{l}}$, $l=1,\ldots,T$.}
\end{equation*}
Correspondingly, the upscaling operator is formulated as
\begin{equation*}
\psi(\rvy) = \hat{\rvx}^{(0)},
\quad
\text{where $\hat{\rvx}^{(T)} = \rvy$, $\hat{\rvx}^{(l-1)} = \E_{p(\rvz^{(l)})}[\mW^{\top} (f^{(l)})^{-1}(\hat{\rvx}^{(l)}, \rvz^{(l)})]$, $l=T, \ldots, 1$.}
\end{equation*}
\cref{fig:lr2flow_arch} illustrates the computational flow of LR2Flow.

\begin{figure}
\centering
\includegraphics[width=\linewidth]{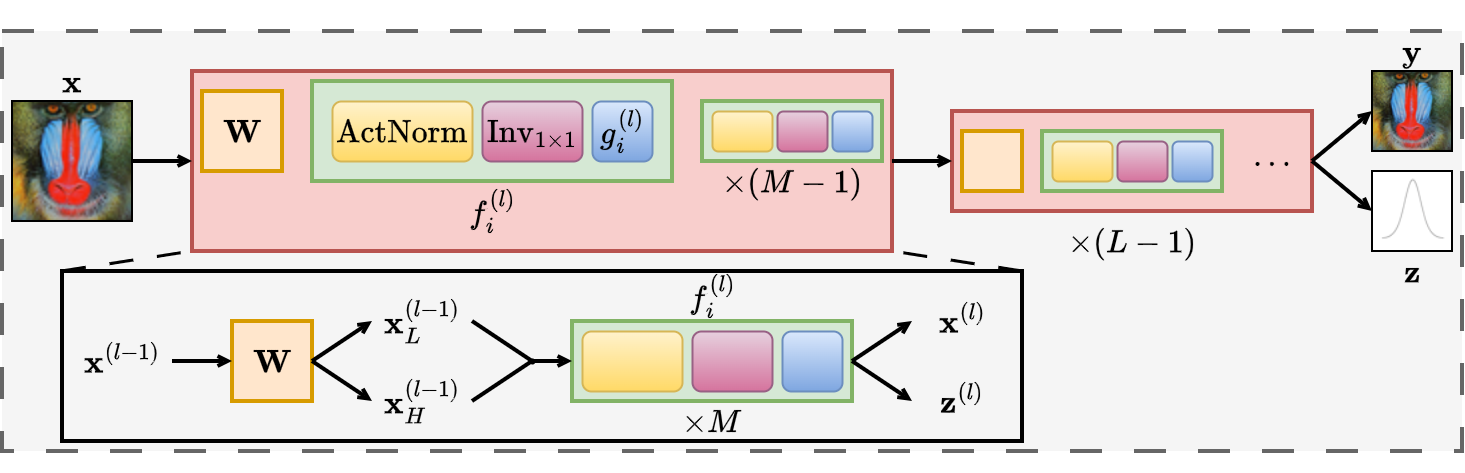}
\caption{Computational flow of LR2Flow. Here, $\rvz = (\rvz^{(1)}, \ldots, \rvz^{(T)})$ denotes the concatenation of latent variables across all scale levels.}
\label{fig:lr2flow_arch}
\end{figure}

\subsection{Reconstruction Error Analysis}
\label{sec:error_analysis}
Let $q(\rvx)$ denote the distribution of the original HR images, with covariance $\mSigma = \Cov[\rvx]$, and let the latent prior be given by $p = \gN(\mathbf{0}, \sigma^2\mI)$ with a temperature parameter $\sigma \geq 0$. This section presents a theoretical analysis of the reconstruction error by deriving an upper bound for the following quantity:
\begin{equation}
\label{eq:reconstruction_loss_inn}
 \rve^* = \inf_{\gF\in\Theta}\big\{
\E_{q(\rvx)p(\rvz)}\left\|\rvx - \mW^{\top}\gF^{-1}\left([\gF(\mW\rvx)]_{1:d}, \rvz\right)
\right\|^2\big\},
\end{equation}
where $\mW = (\mW_L^{\top}, \mW_H^{\top})^{\top}$ satisfies $\mW^{\top}\mW=\mI$, $\Theta$ denotes the hypothesis space of invertible neural networks. The quantity $\rve^*$ characterizes the reconstruction capability of different architectural choices. In the subsequent analysis, we focus on two representative classes of invertible architectures: the affine coupling layers defined in~\eqref{eq:affine_coupling} and the invertible residual blocks described in~\eqref{eq:iresblock}.

{\bf Case I: affine coupling layers.} 
As defined in~\eqref{eq:affine_coupling}, we consider a one-layer affine coupling structure. 
Then, the corresponding hypothesis space is given by
\begin{equation}
\label{eq:affine_coupling_class}
\Theta 
= \left\{
\gF:\ (\rvx_L, \rvx_H) \mapsto 
\bigl(\rvx_L,\ \rho(\rvx_L) \odot \rvx_H + \eta(\rvx_L)\bigr)
\;\middle|\;
\rho, \eta:\mathbb{R}^{d}\to\mathbb{R}^{N-d}
\right\}.
\end{equation}
The following proposition characterizes the optimal reconstruction error $\rve^*$, defined in \eqref{eq:reconstruction_loss_inn}.

\begin{proposition}
\label{prop:single_affine_coupling_case}
Let $\Theta$ be defined as in~\eqref{eq:affine_coupling_class}. Then, the minimal reconstruction error, defined in \eqref{eq:reconstruction_loss_inn}, satisfies
\begin{equation}
\label{eq:optim_single_affine}
\rve^*=\E_{q(\rvx)}\Tr\left(\Var\left[\left.\mW_H\rvx \right| \mW_L\rvx\right]\mW_H\mW_H^{\top}\right).
\end{equation}
The optimal solution $\rho^*$ and $\eta^*$ from the hypothesis space \eqref{eq:affine_coupling_class} satisfies:
$\eta^*\left(\rvx_L\right) = -\rho^*\left(\rvx_L\right) \odot \E[\mW_H\rvx |\mW_L\rvx=\rvx_L]$.
\end{proposition}
The proof of \cref{prop:single_affine_coupling_case} is given in \cref{appendix:affine_coupling_beyond}.  It is worth noting that the optimal value in \eqref{eq:optim_single_affine} is equivalent to $\E_{q(\rvx)}\|\rvx - \E[\rvx | \mW_L \rvx]\|^2$. This implies that \emph{an appropriately selected wavelet system can effectively capture the principal components of the data distribution within its low-frequency subband}.


Next, we extend the above result to a broader family of invertible maps:
\begin{equation}
\label{eq:extension_affine_coupling}
\begin{aligned}
\Theta:= \left\{\gF:(\rvx_L,\rvx_{H}) \mapsto (\rvx_{L}, h(\rvx_{H}; \rvx_{L})) \mid h(\cdot;\rvx_L) \text{ is invertible on $\sR^{N-d}$}\right\}.
\end{aligned}
\end{equation}
If $\Theta$ is given by \eqref{eq:extension_affine_coupling}, we show an upper bound of \eqref{eq:reconstruction_loss_inn}.

\begin{proposition}
\label{prop:case2_reconstruction_error}
Let $\Theta$ be defined as in~\eqref{eq:extension_affine_coupling}. The minimal reconstruction error, defined in \eqref{eq:reconstruction_loss_inn}, has the following form:
\begin{equation}
    \label{eqn:upperbound-e}
    \rve^*\leq \Tr\big( (\mW_H^{\top}\mW_H)^2 (\mI - \mW_L^{\dagger}\mW_L)\mSigma(\mI - \mW_L^{\dagger}\mW_L)\big),
\end{equation}
where $\mW_L^\dagger$ is the Moore-Penrose inverse of $\mW_L$.
\end{proposition}
We defer the proof of \cref{prop:case2_reconstruction_error} to the \cref{appendix:affine_coupling_beyond}.

\begin{remark}
\label{remark:affine_coupling_extension}
For a symmetric matrix $\mA$, let $\lambda^{\downarrow}_{i}(\mA)$ and $\lambda^{\uparrow}_{i}(\mA)$ denote the $i$-th largest and $i$-th smallest eigenvalues of $\mA$, respectively, and let $\mV^{\downarrow}_{i}(\mA)$ denote the eigenspace associated with $\lambda^{\downarrow}_{i}(\mA)$. Let $J(\mW)$ denote the upper bound in~\eqref{eqn:upperbound-e}. Then, 
we have 
\begin{equation}
\label{eqn:bound}
\inf_{\mW}\,\{J(\mW) \mid \mW^\top\mW=\mI\}=\sum\limits_{i=1}^{n-d}\bigl(\lambda^{\uparrow}_{i}(\mW_H^{\top}\mW_H)\bigr)^2 \cdot \lambda^{\downarrow}_{i+d}(\mSigma).
\end{equation}
Moreover, this minimum is attained when the low-frequency subspace aligns with the principal eigenspaces of the data covariance, i.e., $
\mathrm{Im}(\mW_L^{\top}) \;=\; \bigoplus_{i=1}^{d} \mV^{\downarrow}_{i}(\mSigma).
$
Refer to \cref{appendix:affine_coupling_beyond} for the proof of \eqref{eqn:bound}. 
This characterization provides a criterion for constructing $\mW$ based on the eigendecomposition of the data covariance $\mSigma=\mathrm{Cov}[\rvx]$, although accurately estimating $\mSigma$ in practice may be challenging.
\end{remark}

{\bf Case II: invertible residual blocks (iResBlocks).}
We consider the one-layer iResBlock architecture defined in~\eqref{eq:iresblock}, with the hypothesis space specified by
\begin{equation}
\label{eq:iresblock_set}
\Theta = \{\gF:\rvc\mapsto\rvc + \phi(\rvc) \mid \text{$\phi:\sR^{N}\to\sR^{N}$, Lip$(\phi) \leq L < 1$}\}.
\end{equation}
The following proposition establishes an upper bound on the optimal reconstruction error $\rve^*$ defined in~\eqref{eq:reconstruction_loss_inn}.

\begin{proposition}
\label{prop:iresnet_optiom_reconstruction_error}
Let $\Theta$ be defined as in~\eqref{eq:iresblock_set}. Then, the minimal reconstruction error defined in~\eqref{eq:reconstruction_loss_inn} satisfies the following bound:
\begin{equation}
\label{eq:iresnet_bound2}
\rve^* \leq \frac{\inf_{f\in\gK} \sqrt{\E_{q(\rvx)}\|\left(\vone_{d} + f(\mW_L\rvx)\right)^2\|^2}}{(1-L)^2} \cdot \sqrt{\E_{q(\rvx)}\|\Var[\mW_H\rvx | \mW_L\rvx]\|_{F}^2},
\end{equation}
where $\mathbf{1}_{d}\in\sR^{d}$ denotes the vector of all ones, and $\gK = \{f:\sR^{d}\to\sR^{d} \mid \text{Lip}(f) \leq L\}$.
\end{proposition}

We detail the proof in \cref{appendix:proof_prop_iresnet_optim_reconstruction_error}. Our analysis reveals that the upper bound of the optimal reconstruction error in~\eqref{eq:reconstruction_loss_inn} for the iResBlock architecture is primarily governed by the term
$\E_{q(\rvx)}\|\Var[\mW_H\rvx | \mW_L\rvx]\|_F$. 
Consequently, designing an effective wavelet system for iResBlock architectures necessitates minimizing the conditional variance of the high-frequency subband. This requirement implies that \emph{the low-frequency subband must retain the principal components of the data distribution}, a conclusion that aligns with our findings for the affine coupling layer.

Additionally, we analyze the case where $\Theta = O(N) = \{\mF\in\sR^{N\times N} \mid \mF\mF^{\top} = \mF^{\top}\mF = \mI\}$.
This analysis, detailed in \cref{appendix:proof_linear_orthogonal}, highlights the inherent limitations of linear invertible mappings.

\begin{remark}
The parameterization family $\Theta$ considered in~\eqref{eq:affine_coupling_class} and~\eqref{eq:iresblock_set}
consists of a single block, motivating the choice of the basic building block in LR2Flow. We evaluate these single-block networks with different wavelet transforms in \cref{tab:ablation_wavelet_toy}. Based on the results, LR2Flow employs a composition of linear B-spline wavelet tight frames and multiple affine coupling layers. We note that extending the theoretical analysis to the multi-layer setting is nontrivial, as it requires analyzing the composition of multiple trainable invertible mappings.
\end{remark}

\section{Applications of LR2Flow}
In this section, we introduce the applications of LR2Flow in three image processing tasks.

\subsection*{Image rescaling}
Image rescaling comprises two fundamental processes: downscaling a HR image to a LR version to reduce storage and transmission costs, and upscaling the LR image to restore the original resolution. The primary objective is to preserve visual appearance and detail fidelity while adjusting the resolution. This goal aligns closely with the task of learning effective LR representations, making LR2Flow a suitable framework. 
During downscaling, we generate the LR representation $\rvy$ from the HR image $\rvx$ via $\rvy = \varphi(\rvx)$. Conversely, upscaling yields the HR reconstruction $\hat{\rvx}$ via $\hat{\rvx} = \psi(\rvy)$. Inspired by~\cite{xiao2020invertible,xiao2023invertible,liang2021hierarchical}, our training objective incorporates three components: (i) an HR reconstruction loss $\gL_{\text{HR}} = \|\hat{\rvx} - \rvx\|_1$ that ensures fidelity between the reconstructed and original HR images; (ii) an LR guidance loss $\gL_{\text{LR}} = \|\rvy - \text{Bic}(\rvx)\|_2^2$ that aligns $\rvy$ with a Bicubic-downsampled reference to maintain visual consistency, where $\text{Bic}:\sR^{n}\to\sR^{d}$ denotes the Bicubic interpolation operator; and (iii) a distribution matching loss $\gL_{\text{dist}} = \|\rvz\|_2^2$ that aligns the model-induced latent distribution $\left([\gF\circ \mW]_{(d+1):N}\right)_{\#}[q]$ with the prior $p = \gN(\mathbf{0},\sigma^2\mI)$ by minimizing their cross-entropy. Here, $\rvz = [\gF(\mW\rvx)]_{(d+1):N}$ denotes the high-frequency latent variable.
The overall training objective for image rescaling is defined as
\begin{equation}
\label{eq:rescaling_loss}
\gL = \lambda_{\text{HR}}\gL_{\text{HR}} + \lambda_{\text{LR}}\gL_{\text{LR}} + \lambda_{\text{dist}} \gL_{\text{dist}},
\end{equation}
where $\lambda_{\text{HR}}$, $\lambda_{\text{LR}}$, and $\lambda_{\text{dist}}$ are hyperparameters.

\subsection*{Image compression}
While image rescaling inherently conserves storage by reducing spatial resolution, applying standard lossy compression (e.g., JPEG) to the resulting LR image enables significantly higher compression rates. In this context, image compression can still be modeled as learning an LR representation, where the goal is to reconstruct the original image $\rvx$ from the compressed representation ${\rvy} = \text{JPEG}(\varphi(\rvx))$, despite the severe loss of high-frequency details. The training objective for compression task follows the formulation in~\eqref{eq:rescaling_loss}. However, since the standard JPEG operator is non-differentiable, we employ a differentiable JPEG simulator~\cite{shin2017jpeg,xing2023scale} during training to facilitate gradient backpropagation.

We further introduce a rate-distortion (R-D) balancing variant LR2Flow(R), by incorporating a differentiable bitrate proxy.
Following~\cite{balle2016end,balle2018variational}, we employ a factorized Gaussian entropy model $\pi_{\bm \omega}$ to estimate the coding cost
\begin{equation*}
{\textstyle \gL_{\text{rate}} = \E_{\rvy \sim (\mathrm{JPEG} \circ \varphi)_{\#}q} \big[-\frac{1}{n}\sum_{i=1}^{d} \log_{2} \pi_{\bm\omega}(\rvy_{i}) \big]},
\end{equation*}
where $\pi_{\bm\omega}(\rvy_{i}) = \int_{\rvy_{i} - \Delta/2}^{\rvy_{i}+ \Delta/2}
\gN(t | {\mu},{\sigma}^2)\,\df t$, $\Delta = \frac{1}{255}$, and ${\bm\omega} = \{{\mu},{\sigma}^2\}$ is the trainable parameters of the entropy model.
The training objective for image compression with the R-D trade-off is defined as
\begin{equation}
\label{eq:comp_r-d}
\gL_{\text{comp}} = \gL + \lambda_{\text{rate}} \gL_{\text{rate}},
\end{equation}
where $\gL$ is the rescaling loss defined in~\eqref{eq:rescaling_loss}, and $\lambda_{\mathrm{rate}}$ is a hyperparameter.

\subsection*{Image denoising}
\label{sec:application_denoise}
We extend LR2Flow to image denoising by interpreting the framework as a \emph{refined frequency-domain decomposition method}: $(\rvy, \rvz) = (\gF \circ \mW)(\rvx)$. Here, $\rvy$ denotes the refined low-frequency component and $\rvz$ the high-frequency residual, with synthesis given by $\mW^{\top} \circ \gF^{-1}$.  Leveraging the observation that \emph{clean and noisy image pairs exhibit strong consistency in the low-frequency subband~\cite{ruderman1993statistics,field1987relations}}, we optimize LR2Flow to align the LR representations of the clean image $\rvx_{\rvc}$ and the noisy image $\rvx_{\rvn} = \rvx_{\rvc} + \rvn$. Formally, given the decompositions $(\rvy_{\rvc}, \rvz_{\rvc}) = (\gF \circ \mW)(\rvx_{\rvc})$ and $(\rvy_{\rvn}, \rvz_{\rvn}) = (\gF \circ \mW)(\rvx_{\rvn})$, our training objective is to ensure that $\rvy_{\rvc} \approx \rvy_{\rvn}$.

A naive denoising strategy involves modeling $\rvz_{\rvc}$ using a prior $p$ and simply replacing $\rvz_{\rvn}$ with a sample from $p$ during inference. However, disregarding the specific values of $\rvz_{\rvn}$ ignores the structural information it contains, often resulting in over-smoothed reconstructions. To address this, we recover high-frequency details via $\hat{\rvz} = \gR(\rvz_{\rvn}; \rvy_{\rvn})$, employing a restoration network $\gR:\sR^{N-d}\times \sR^{d} \to \sR^{N-d}$. The final denoised image is then synthesized as $\hat{\rvx} = (\mW^{\top} \circ \gF^{-1})(\rvy_{\rvn}, \hat{\rvz})$.

The training objective comprises three terms: (i) the image reconstruction loss $\gL_{\text{img}} = \|\hat{\rvx} - \rvx_{\rvc}\|_{1}$, which ensures fidelity; (ii) the low-frequency alignment loss $\gL_{\text{lf}} = \|\rvy_{\rvn} - \rvy_{\rvc}\|_2^2$, designed to enforce consistency in the LR representations; and (iii) the high-frequency restoration loss $\gL_{\text{hf}} = \|\rvz_{\rvc} - \hat{\rvz}\|_2^2$.  The total loss is defined as
\begin{equation}
\label{eq:restoration_loss}
\gL_{\text{denoising}} = \lambda_{\text{img}}\gL_{\text{img}} + \lambda_{\text{lf}}\gL_{\text{lf}} + \lambda_{\text{hf}}\gL_{\text{hf}},
\end{equation}
where $\lambda_{\text{img}}, \lambda_{\text{lf}}$, and $\lambda_{\text{hf}}$ are weighting parameters.

\begin{remark}
We apply the tensor formulation of wavelet tight frame to the two dimensional color images. The detailed construction and implementation are provided in~\cref{sec:2d_construction}.
\end{remark}

\section{Experiments}
We evaluate LR2Flow on image rescaling, image compression, and image denoising to demonstrate its effectiveness and versatility. All experiments are conducted on a single NVIDIA H800 GPU. For the architectural configuration, we set the hierarchy level to $T=1$ for image compression and denoising. For image rescaling, we set $T = \log_{2} s$, where $s$ denotes the rescaling factor. Across all configurations, each flow level $f^{(l)}$ consists of $M=8$ flow blocks, with affine coupling layers used to implement the invertible transformation $g^{(l)}_i$.
Detailed implementation of the network architecture is provided in~\cref{sec:net_impl}.
For all tasks, we set the temperature parameter $\sigma$ of the latent prior $p=\gN(\mathbf{0},\sigma^{2}\mI)$ to $\sigma=1$ during training and use one latent sample when computing the reverse transform. During inference, we adopt a deterministic setting by fixing $\sigma=0$.
In the reported results, the best scores are highlighted in \textbf{bold}, and the second-best scores are \underline{underlined}.

\subsection{Image Rescaling}

We evaluate image rescaling at rescaling factors $s=2$ and $s=4$.
For both factors, the model is trained on the DIV2K~\cite{agustsson2017ntire} dataset for 500k iterations using the AdamW optimizer~\cite{loshchilov2017decoupled} with $\beta_1=0.9$, $\beta_2=0.99$, and zero weight decay. The learning rate is initialized at $2 \times 10^{-4}$ and reduced by half at iterations 100k, 200k, 300k, 350k, 400k, and 450k.  We utilize a batch size of 16 with randomly cropped $160 \times 160$ patches and apply standard geometric augmentations, including horizontal flips and rotations. The loss weights in~\eqref{eq:rescaling_loss} are set to $\lambda_{\text{HR}} = 1.0$, $\lambda_{\text{LR}} = 5 \times 10^{-2}$, and $\lambda_{\text{dist}} = 10^{-5}$. Model performance is reported on the DIV2K validation set and four standard benchmarks: Set5~\cite{bevilacqua2012low}, Set14~\cite{zeyde2012single}, BSD100~\cite{martin2001database}, and Urban100~\cite{huang2015single}. Following established protocols~\cite{xiao2020invertible,xiao2023invertible,liang2021hierarchical}, we compute PSNR and SSIM metrics on the Y channel of the YCbCr color space to assess HR image reconstruction.

\begin{table*}[ht]
\caption{Quantitative comparison of image rescaling performance (PSNR/SSIM) on benchmark datasets.}
\label{tab:quan_rescaling}
\centering
\resizebox{\linewidth}{!}{

\begin{tabular}{c|l|cc|cc|cc|cc|cc}
\Xhline{2\arrayrulewidth}
\multirow{2}{*}{Scale} & \multirow{2}{*}{\makecell[l]{Downscaling \& \\ Upscaling}} & \multicolumn{2}{c|}{Set5} & \multicolumn{2}{c|}{Set14} & \multicolumn{2}{c|}{BSD100} & \multicolumn{2}{c|}{Urban100} & \multicolumn{2}{c}{DIV2K} \\
\cline{3-12}
& & PSNR & SSIM & PSNR & SSIM & PSNR & SSIM & PSNR & SSIM & PSNR & SSIM \\
\Xhline{2\arrayrulewidth}
\multirow{14}{*}{$\times2$} & {Bicubic \& Bicubic} & 33.66 & 0.9299 & 30.24 & 0.8688 & 29.56 & 0.8431 & 26.88 & 0.8403 & 31.01 & 0.9393 \\
& {Bicubic \& SRCNN~\cite{dong2015image}} & 36.66 & 0.9542 & 32.45 & 0.9067 & 31.36 & 0.8879 & 29.50 & 0.8946 & - & - \\
& Bicubic \& CARN~\cite{ahn2018fast} & 37.76 & 0.9590 & 33.52 & 0.9166 & 32.09 & 0.8978 & 31.92 & 0.9256 & - & - \\
& Bicubic \& EDSR~\cite{lim2017enhanced} & 38.20 & 0.9606 & 34.02 & 0.9204 & 32.37 & 0.9018 & 33.10 & 0.9363 & 35.12 & 0.9699 \\
& Bicubic \& RCAN~\cite{zhang2018image} & 38.27 & 0.9614 & 34.12 & 0.9216 & 32.41 & 0.9027 & 33.34 & 0.9384 & - & - \\
& Bicubic \& SAN~\cite{dai2019second} & 38.31 & 0.9620 & 34.07 & 0.9213 & 32.42 & 0.9028 & 33.10 & 0.9370 & - & - \\
& Bicubic \& SwinIR~\cite{liang2021swinir} & 38.42 & 0.9623 & 34.46 & 0.9250 & 32.53 & 0.9041 & 33.81 & 0.9427 & - & - \\
& Bicubic \& HAT~\cite{chen2023activating} & 38.91 & 0.9646 & 35.29 & 0.9293 & 32.74 & 0.9066 & 35.09 & 0.9505 & - & - \\
\cline{2-12}
& TAD \& TAU~\cite{kim2018task} & 38.46 & - & 35.52 & - & 36.68 & - & 35.03 & - & 39.01 & - \\
& CAR \& EDSR~\cite{sun2020learned} & 38.94 & 0.9658 & 35.61 & 0.9404 & 33.83 & 0.9262 & 35.24 & 0.9572 & 38.26 & 0.9599 \\
& IRN~\cite{xiao2020invertible} & 43.99 & 0.9871 & 40.79 & 0.9778 & 41.32 & 0.9876 & 39.92 & 0.9865 & 44.32 & 0.9908 \\
& T-IRN~\cite{bao2025tinvblock} & 44.86 & 0.9883 & 41.70 & 0.9809 & \underline{42.68} & \underline{0.9913} & 41.05 & 0.9899 & 45.46 & 0.9932 \\
& HCFlow~\cite{liang2021hierarchical} & \underline{45.08} & \underline{0.9895} & \underline{42.30} & \underline{0.9827} & 42.61 & 0.9909 & \underline{41.92} & \underline{0.9928} & \underline{45.66} & \underline{0.9933} \\
& LR2Flow & \textbf{46.87} & \textbf{0.9932} & \textbf{43.98} & \textbf{0.9880} & \textbf{44.94} & \textbf{0.9948} & \textbf{43.56} & \textbf{0.9951} & \textbf{47.55} & \textbf{0.9958} \\
\hline
\multirow{14}{*}{$\times4$} & Bicubic \& Bicubic & 28.42 & 0.8104 & 26.00 & 0.7027 & 25.96 & 0.6675 & 23.14 & 0.6577 & 26.66 & 0.8521 \\
& Bicubic \& SRCNN~\cite{dong2015image} & 30.48 & 0.8628 & 27.50 & 0.7513 & 26.90 & 0.7101 & 24.52 & 0.7221 & - & - \\
& Bicubic \& CARN~\cite{ahn2018fast} & 32.13 & 0.8937 & 28.60 & 0.7806 & 27.58 & 0.7349 & 26.07 & 0.7837 & - & - \\
& Bicubic \& EDSR~\cite{lim2017enhanced} & 32.46 & 0.8968 & 28.80 & 0.7760 & 27.71 & 0.7420 & 26.64 & 0.8033 & 29.38 & 0.9032 \\
& Bicubic \& RCAN~\cite{zhang2018image} & 32.63 & 0.9002 & 28.87 & 0.7889 & 27.77 & 0.7436 & 26.82 & 0.8087 & 30.77 & 0.8460 \\
& Bicubic \& SAN~\cite{dai2019second} & 32.64 & 0.9003 & 28.92 & 0.7888 & 27.78 & 0.7436 & 26.79 & 0.8068 & - & - \\
& Bicubic \& SwinIR~\cite{liang2021swinir} & 32.92 & 0.9044 & 29.09 & 0.7950 & 27.92 & 0.7489 & 27.45 & 0.8254 & - & - \\
& Bicubic \& HAT~\cite{chen2023activating} & 33.30 & 0.9083 & 29.47 & 0.8015 & 28.09 & 0.7551 & 28.60 & 0.8498 & - & - \\
\cline{2-12}
& TAD \& TAU~\cite{kim2018task} & 31.81 & - & 28.63 & - & 28.51 & - & 26.63 & - & 31.16 & - \\
& CAR \& EDSR~\cite{sun2020learned} & 33.88 & 0.9174 & 30.31 & 0.8382 & 29.15 & 0.8001 & 29.28 & 0.8711 & 32.82 & 0.8837 \\
& IRN~\cite{xiao2020invertible} & 36.19 & 0.9451 & 32.67 & 0.9015 & 31.64 & 0.8826 & 31.41 & 0.9157 & 35.07 & 0.9318 \\
& T-IRN~\cite{bao2025tinvblock} & \underline{36.29} & 0.9452 & 32.70 & 0.9003 & 31.64 & 0.8837 & 31.19 & 0.9132 & 35.10 & 0.9328 \\
& HCFlow~\cite{liang2021hierarchical} & \underline{36.29} & \underline{0.9468} & \underline{33.02} & \underline{0.9065} & \underline{31.74} & \underline{0.8864} & \underline{31.62} & \underline{0.9206} & \underline{35.23} & \underline{0.9346} \\
& LR2Flow & \textbf{37.00} & \textbf{0.9493} & \textbf{33.89} & \textbf{0.9121} & \textbf{32.33} & \textbf{0.8901} & \textbf{33.37} & \textbf{0.9352} & \textbf{35.97} & \textbf{0.9383} \\
\Xhline{2\arrayrulewidth}
\end{tabular}
}
\end{table*}

\textbf{Results.} We evaluate LR2Flow against three categories of rescaling approaches: (i) standard SR approaches using fixed Bicubic downscaling~\cite{dong2015image,ahn2018fast,lim2017enhanced,zhang2018image,dai2019second,liang2021swinir,chen2023activating}; (ii) encoder-decoder frameworks that jointly learn downscaling and upscaling~\cite{kim2018task,sun2020learned}; and (iii) flow-based rescaling models~\cite{xiao2020invertible,liang2021hierarchical,bao2025tinvblock}.

Quantitative results presented in \cref{tab:quan_rescaling} indicate that SR-only pipelines, including advanced Transformer baselines such as SwinIR and HAT, exhibit limited performance due to their reliance on fixed downscaling kernels. This result underscores the value of task-aware or content-adaptive downscaling. Although encoder-decoder methods that jointly optimize downscaling and upscaling deliver substantial gains, they often suffer from underconstrained feature learning due to insufficient regularization. In contrast, flow-based models leverage invertibility and density estimation to explicitly model high-frequency residuals. This capability reduces information loss and improves reconstruction quality.


\begin{figure*}
\caption{Qualitative comparison of $\times 4$ image rescaling results. Representative examples are selected from the Set14, BSD100, and Urban100 datasets.}
\label{fig:visual_rescaling}
\centering
\includegraphics[width=\linewidth]{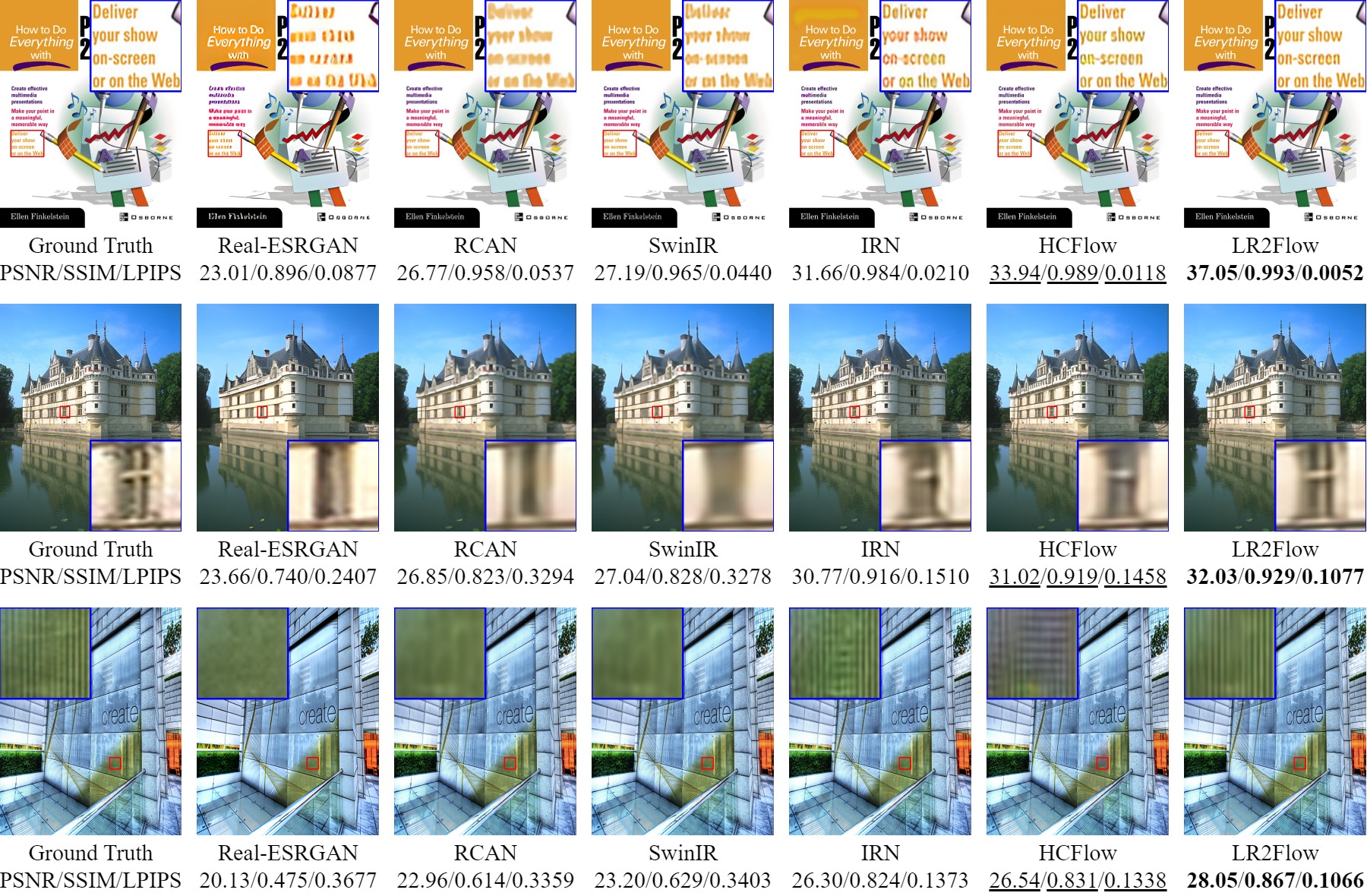}
\end{figure*}

Our proposed LR2Flow significantly enhances HR reconstruction, achieving superior PSNR and SSIM scores compared to prior methods. As illustrated in \cref{fig:visual_rescaling}, LR2Flow preserves edges and textures more faithfully, effectively reducing blur and distortion. We attribute these performance gains to the enhanced stability provided by the redundant wavelet tight frame.


\subsection{Image Compression}
Our compression pipeline combines the $\times2$ rescaling model with lossy JPEG compression. To exploit the structural similarity between the LR representations learned for rescaling and compression, we initialize the compression model with weights pretrained on the $\times2$ rescaling task.
We fine-tune the model for $200$k iterations using the AdamW optimizer with $\beta_1 = 0.9$, $\beta_2 = 0.99$, and weight decay set to $0$. The learning rate is initialized to $2 \times 10^{-5}$ and halved at iterations $50$k, $100$k, $150$k, and $175$k. Training uses a batch size of $16$, with randomly cropped $160 \times 160$ patches and standard geometric augmentations.
The training objective for LR2Flow follows~\eqref{eq:rescaling_loss}, with weights set to $\lambda_{\text{HR}} = 1.0$, $\lambda_{\text{LR}} = 5 \times 10^{-2}$, and $\lambda_{\text{dist}} = 10^{-5}$. 
For the R-D trade-off variant LR2Flow(R), the objective follows~\eqref{eq:comp_r-d}, with $\lambda_{\text{rate}} = 3\times 10^{-3}$.
During training, the JPEG quality factor (QF) is uniformly sampled from $\{50, 55, \dots, 90\}$. For evaluation, we report PSNR and SSIM on the Y channel of the YCbCr color space for the Set5, Set14, BSD100, Urban100, and DIV2K validation sets at QF values of $\{30, 50, 70, 80, 90\}$.

\begin{table*}
\centering
\caption{Quantitative comparison of rescaling-based compression methods. The rescaling factor is set to $s=2$.}
\label{tab:quan_compression}
\resizebox{\linewidth}{!}{
\begin{tabular}{l|cc|cc|cc|cc|cc}
\Xhline{2\arrayrulewidth}
\multirow{2}{*}{\makecell[l]{Downscaling \& \\ Upscaling}} & \multicolumn{2}{c|}{JPEG QF=30} & \multicolumn{2}{c|}{JPEG QF=50} & \multicolumn{2}{c|}{JPEG QF=70} & \multicolumn{2}{c|}{JPEG QF=80} & \multicolumn{2}{c}{JPEG QF=90} \\
\cline{2-11}
& PSNR & SSIM & PSNR & SSIM & PSNR & SSIM & PSNR & SSIM & PSNR & SSIM \\
\Xhline{2\arrayrulewidth}
Bicubic \& Bicubic & 30.11 & 0.8296 & 30.75 & 0.8494 & 31.30 & 0.8659 & 31.66 & 0.8765 & 32.12 & 0.8904 \\
Bicubic \& EDSR~\cite{lim2017enhanced} & 27.58 & 0.8166 & 27.88 & 0.8352 & 28.16 & 0.8512 & 28.34 & 0.8616 & 28.55 & 0.8755 \\
Bicubic \& RCAN~\cite{zhang2018image} & 28.57 & 0.8158 & 29.00 & 0.8346 & 29.34 & 0.8505 & 29.56 & 0.8609 & 29.85 & 0.8746 \\
Bicubic \& SwinIR~\cite{liang2021swinir} & 30.65 & 0.8400 & 31.56 & 0.8633 & 32.45 & 0.8838 & 33.11 & 0.8978 & 34.17 & 0.9171 \\
IRN~\cite{xiao2020invertible} & 29.11 & 0.8133 & 29.60 & 0.8334 & 30.00 & 0.8509 & 30.25 & 0.8621 & 30.56 & 0.8767 \\
\hline
IRN + CRM~\cite{xiao2023invertible} & 30.41 & 0.8347 & 31.61 & 0.8637 & 32.13 & 0.8784 & 32.76 & 0.8925 & 33.67 & 0.9114 \\
SAIN~\cite{yang2023self} & \underline{31.47} & \underline{0.8747} & \underline{33.17} & \underline{0.9082} & \underline{34.73} & \underline{0.9296} & \underline{35.46} & \underline{0.9374} & \underline{35.96} & \underline{0.9419} \\
LR2Flow & \textbf{32.37} & \textbf{0.8968} & \textbf{33.98} & \textbf{0.9217} & \textbf{35.15} & \textbf{0.9353} & \textbf{35.71} & \textbf{0.9409} & \textbf{36.12} & \textbf{0.9445} \\
\Xhline{2\arrayrulewidth}
\end{tabular}
}
\end{table*}

\textbf{Results.} We compare LR2Flow with two categories of rescaling-based compression pipelines: (i) disjoint approaches that combine frozen rescaling models~\cite{lim2017enhanced,liang2021swinir,zhang2018image,xiao2020invertible} with pretrained JPEG artifact removal networks~\cite{jiang2021towards}; and (ii) end-to-end approaches that jointly optimize the rescaling network and the JPEG simulator~\cite{yang2023self,xiao2023invertible}.


\cref{tab:quan_compression} details the reconstruction performance on the DIV2K validation set. Pipelines utilizing off-the-shelf JPEG decoders with frozen rescaling networks significantly underperform, as fixed upscalers fail to adapt to the specific information loss introduced by lossy compression. While jointly optimizing the rescaling and JPEG decoding narrows this performance gap, such methods typically necessitate auxiliary restoration modules. In contrast, LR2Flow provides a unified framework that learns robust LR representations without requiring auxiliary networks. This design achieves superior performance, particularly at low QFs.
\cref{fig:qualitative_compression} presents a qualitative comparison. 
Disjoint methods often exhibit pronounced blocking and ringing artifacts, whereas competing joint optimization methods tend to over-smooth fine textures.
In contrast, LR2Flow effectively mitigates these artifacts while preserving sharp edges, demonstrating its superior capability in modeling high-frequency details. Furthermore, \cref{fig:quantitative_compression} confirms these improvements across Set5, Set14, BSD100, and Urban100. LR2Flow consistently outperforms competitors across all datasets and QF levels, underscoring its robustness and adaptability.

\begin{figure*}
\caption{Qualitative comparison of compression results on the DIV2K validation set with a rescaling factor $s=2$ and JPEG QF=30.}
\label{fig:qualitative_compression}
\centering
\includegraphics[width=\linewidth]{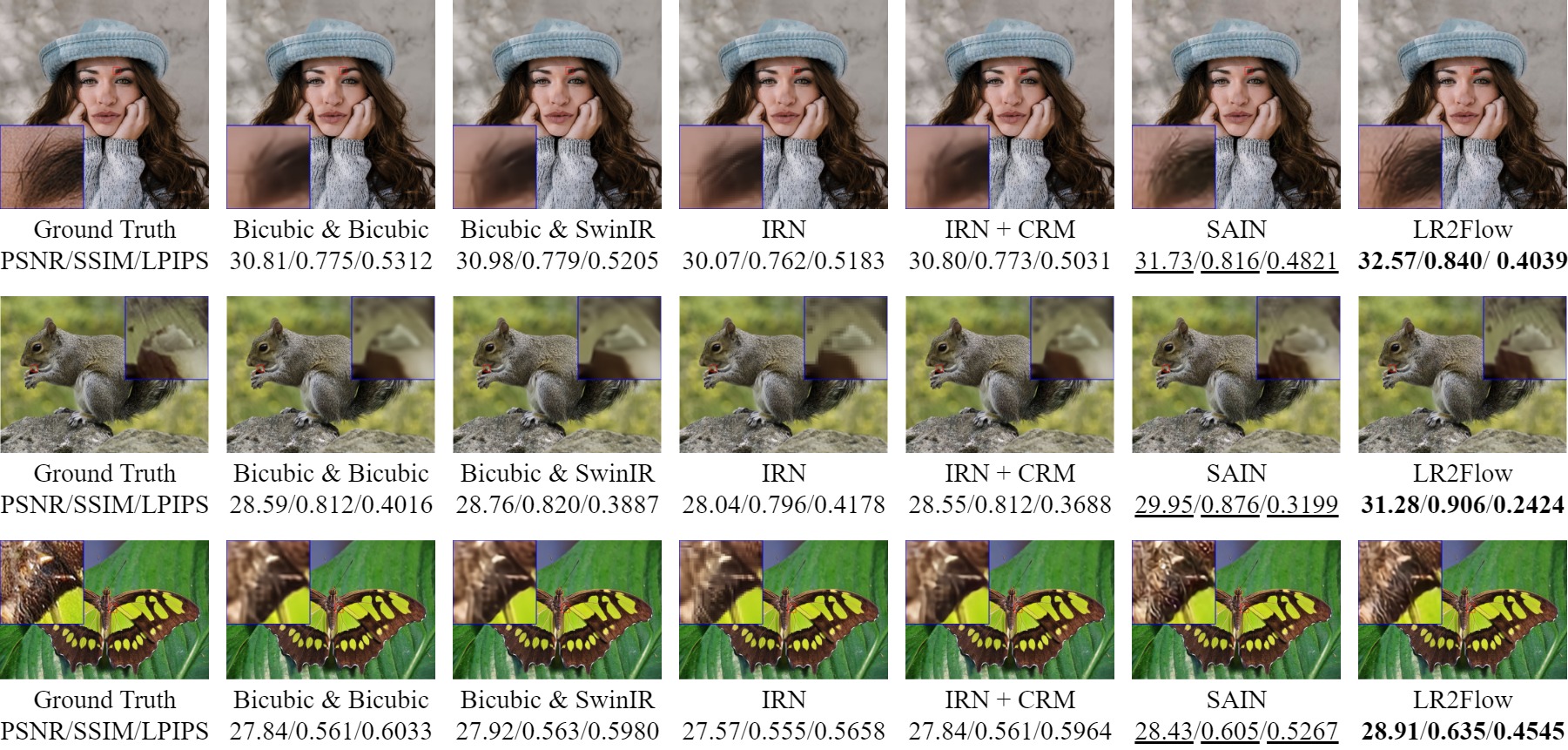}
\end{figure*}


\begin{figure*}
\caption{Quantitative comparison of compression performance with $\times 2$ rescaling across various JPEG QFs on the Set5, Set14, BSD100, and Urban100 benchmarks.}
\label{fig:quantitative_compression}
\centering
\includegraphics[width=\linewidth]{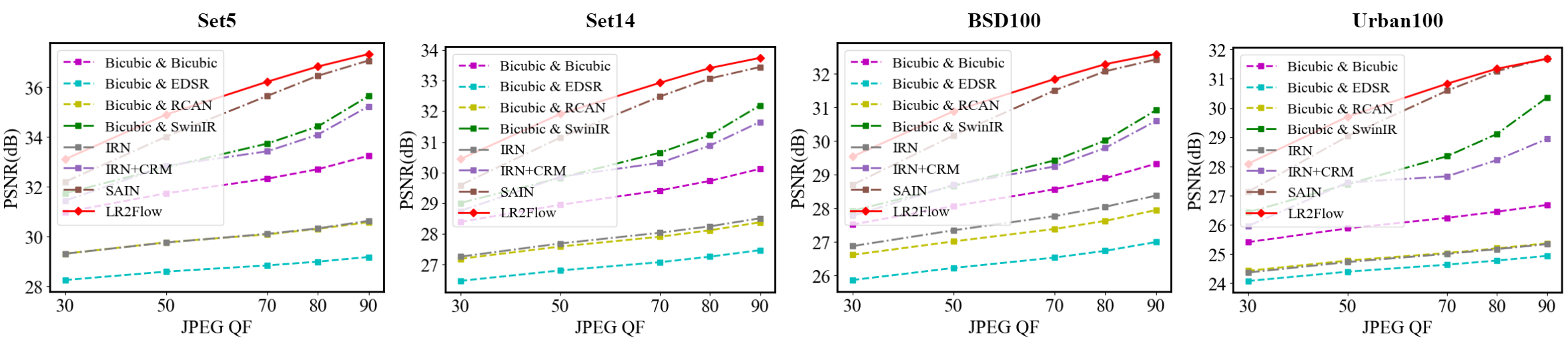}
\end{figure*}


\begin{figure}
\caption{Quantitative evaluation of the R-D trade-off. We report the curves of bpp versus PSNR, SSIM, and LPIPS for the compared compression methods.}
\label{fig:bpp_recon}
\centering
\includegraphics[width=\linewidth]{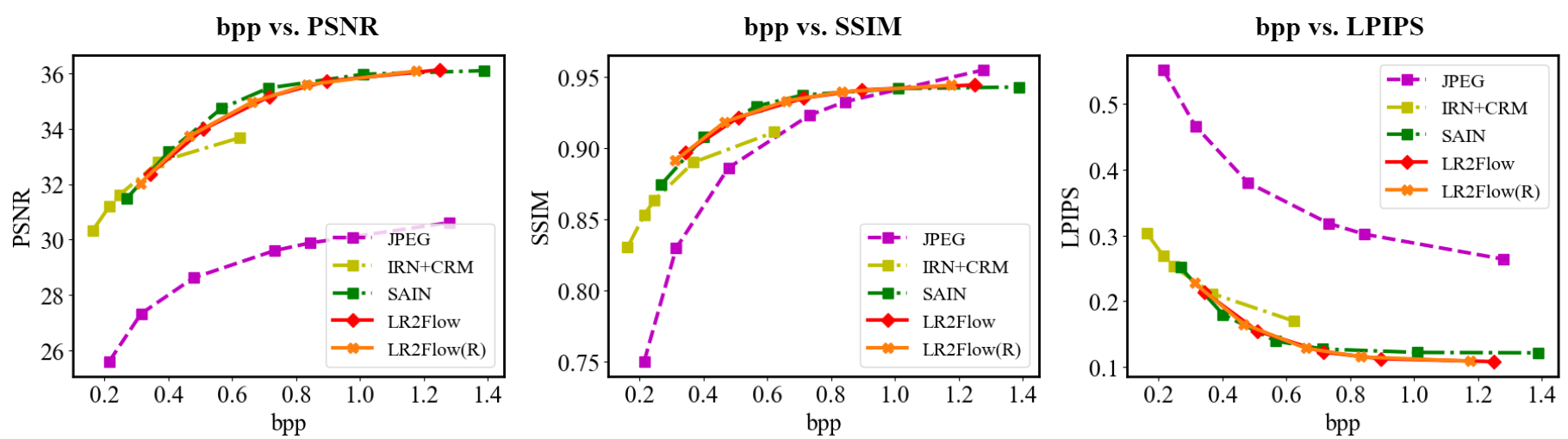}
\end{figure}

We further evaluate LR2Flow and LR2Flow(R) in terms of the R-D trade-off, comparing them with standard JPEG and two leading rescaling-based compression methods: IRN+CRM~\cite{xiao2023invertible} and SAIN~\cite{yang2023self}. Compression efficiency is measured in bits per pixel (bpp). As shown in~\cref{fig:bpp_recon}, LR2Flow achieves competitive overall performance and performs particularly well on perceptual metrics, despite not using the auxiliary JPEG-aware decoders employed by IRN+CRM and SAIN.
LR2Flow(R) further slightly improves the R-D trade-off through rate regularization.
These results demonstrate the potential of our method as a promising approach for high-ratio image compression.
Additional experiments analyzing the R-D balancing behavior of LR2Flow(R) are provided in~\cref{sec:r-d_balance}.


\subsection{Image Denoising}
We evaluate our denoising approach on standard additive white Gaussian noise (AWGN) removal. The model is trained for $1$M iterations on a composite dataset comprising DIV2K, Flickr2K~\cite{lim2017enhanced}, and the Waterloo Exploration Database~\cite{ma2016waterloo}. Optimization is performed with AdamW using $\beta_1 = 0.9$, $\beta_2 = 0.99$, and zero weight decay. The initial learning rate is set to $2\times 10^{-4}$ and halved every $100$k iterations. We use a batch size of $16$, randomly cropped $128 \times 128$ patches, and standard geometric augmentations. The loss weights in~\eqref{eq:restoration_loss} are set to $\lambda_{\text{img}} = 1.0$ and $\lambda_{\text{lf}} = \lambda_{\text{hf}} = 10^{-2}$. Training pairs are synthesized by adding AWGN with noise levels $\sigma_{\rvn}$ sampled uniformly from $[5, 55]$. For evaluation, we report PSNR on CBSD68~\cite{martin2001database}, Kodak24~\cite{kodak24}, McMaster~\cite{zhang2011color}, and Urban100 at $\sigma_{\rvn}\in\{15,25,50\}$, with geometric self-ensemble~\cite{timofte2016seven} applied during inference.

\begin{table*}
\centering
\caption{Quantitative evaluation of denoising performance (PSNR) for various methods on the CBSD68, Kodak24, McMaster, and Urban100 datasets.}
\label{tab:quan_denoising}
\resizebox{\linewidth}{!}
{
\begin{tabular}{l|ccc|ccc|ccc|ccc}
\Xhline{2\arrayrulewidth}
\multirow{2}{*}{Method} & \multicolumn{3}{c|}{CBSD68} & \multicolumn{3}{c|}{Kodak24} & \multicolumn{3}{c|}{McMaster} & \multicolumn{3}{c}{Urban100} \\
\cline{2-13}
& $\sigma_{\rvn}$=15 & $\sigma_{\rvn}$=25 & $\sigma_{\rvn}$=50 & $\sigma_{\rvn}$=15 & $\sigma_{\rvn}$=25 & $\sigma_{\rvn}$=50 & $\sigma_{\rvn}$=15 & $\sigma_{\rvn}$=25 & $\sigma_{\rvn}$=50 & $\sigma_{\rvn}$=15 & $\sigma_{\rvn}$=25 & $\sigma_{\rvn}$=50 \\
\Xhline{2\arrayrulewidth}
IRCNN~\cite{zhang2017learning} & 33.86 & 31.16 & 27.86 & 34.69 & 32.18 & 28.93 & 34.58 & 32.18 & 28.91 & 33.78 & 31.20 & 27.70 \\
FFDNet~\cite{zhang2018ffdnet} & 33.87 & 31.21 & 27.96 & 34.63 & 32.13 & 28.98 & 34.66 & 32.35 & 29.18 & 33.83 & 31.40 & 28.05 \\
DnCNN~\cite{zhang2017beyond} & 33.90 & 31.24 & 27.95 & 34.60 & 32.14 & 28.95 & 33.45 & 31.52 & 28.62 & 32.98 & 30.81 & 27.59 \\
DSNet~\cite{peng2019dilated} & 33.91 & 31.28 & 28.05 & 34.63 & 32.16 & 29.05 & 34.67 & 32.40 & 29.28 & - & - & - \\
DRUNet~\cite{zhang2021plug} & \underline{34.30} & \underline{31.69} & \textbf{28.51} & \underline{35.31} & \underline{32.89} & \underline{29.86} & \textbf{35.40} & \underline{33.14} & \textbf{30.08} & \underline{34.81} & \underline{32.60} & \underline{29.61} \\
LR2Flow & \textbf{34.32} & \textbf{31.70} & \underline{28.50} & \textbf{35.35} & \textbf{32.92} & \textbf{29.88} & \textbf{35.40} & \textbf{33.15} & \textbf{30.08} & \textbf{34.86} & \textbf{32.66} & \textbf{29.65} \\
\hline
RPCNN~\cite{xia2020identifying} & - & 31.24 & 28.06 & - & 32.34 & 29.25 & - & 32.33 & 29.33 & - & 31.81 & 28.62 \\
BRDNet~\cite{tian2020image} & 34.10 & 31.43 & 28.16 & 34.88 & 32.41 & 29.22 & 35.08 & 32.75 & 29.52 & 34.42 & 31.99 & 28.56 \\
RNAN~\cite{zhang2019residual} & - & - & 28.27 & - & - & 29.58 & - & - & 29.72 & - & - & 29.08 \\
RDN~\cite{zhang2021residual} & - & - & 28.31 & - & - & 29.66 & - & - & - & - & - & 29.38 \\
\Xhline{2\arrayrulewidth}
\end{tabular}
}
\end{table*}


\textbf{Results.} We evaluate the denoising performance of LR2Flow against representative CNN baselines under two training regimes: (i) unified models trained to handle multiple noise levels~\cite{zhang2017learning,zhang2018ffdnet,zhang2017beyond,peng2019dilated,zhang2021plug}; and (ii) noise-specific models trained separately for each noise level~\cite{xia2020identifying,tian2020image,zhang2019residual,zhang2021residual}.
Quantitative results are detailed in \cref{tab:quan_denoising}, where LR2Flow exhibits a clear advantage over competing methods. 
\cref{fig:qualitative_denoising} further presents visual comparisons at a noise level of $\sigma_{\rvn}=50$. While standard CNN-based methods tend to over-smooth fine details and introduce hallucinated artifacts under severe noise, LR2Flow preserves faithful colors and sharp textures, demonstrating robust detail recovery.

Furthermore, we compare LR2Flow against classical model-based baselines, including ISTA~\cite{daubechies2004iterative} and TV denoising~\cite{rudin1992nonlinear}, to assess the benefits of its data-adaptive nonlinear mappings over fixed thresholding strategies. Quantitative comparisons are reported in~\cref{tab:quan_denoising_classic}. Both ISTA and TV exhibit a substantial performance gap relative to LR2Flow, highlighting the advantages of data-adaptive representation learning over fixed hand-crafted priors.


\begin{table*}[ht]
\centering
\caption{Quantitative comparison of LR2Flow with classical denoising methods utilizing hand-crafted priors in the image or wavelet domain.}
\label{tab:quan_denoising_classic}
\resizebox{\linewidth}{!}{
\begin{tabular}{ll|ccc|ccc|ccc|ccc}
\Xhline{2\arrayrulewidth}
\multicolumn{2}{l|}{{\multirow{2}{*}{Method}}} & \multicolumn{3}{c|}{CBSD68} & \multicolumn{3}{c|}{Kodak24} & \multicolumn{3}{c|}{McMaster} & \multicolumn{3}{c}{Urban100} \\
\cline{3-14}
& & $\sigma_{\rvn}$=15 & $\sigma_{\rvn}$=25 & $\sigma_{\rvn}$=50 & $\sigma_{\rvn}$=15 & $\sigma_{\rvn}$=25 & $\sigma_{\rvn}$=50 & $\sigma_{\rvn}$=15 & $\sigma_{\rvn}$=25 & $\sigma_{\rvn}$=50 & $\sigma_{\rvn}$=15 & $\sigma_{\rvn}$=25 & $\sigma_{\rvn}$=50 \\
\Xhline{2\arrayrulewidth}
\multicolumn{2}{l|}{Total Variation} & 17.79 & 17.49 & 17.08 & 17.81 & 17.61 & 17.33 & 13.71 & 13.60 & 13.42 & 17.60 & 17.04 & 16.25 \\ 
\hline
\multirow{3}{*}{ISTA} & PixelUnshuffle & 24.84 & 20.54 & 15.03 & 25.15 & 20.96 & 15.53 & 25.15 & 20.96 & 15.53 & 24.90 & 20.64 & 15.13 \\
& Haar & 27.58 & 24.22 & 19.59 & 28.02 & 24.56 & 19.68 & 28.38 & 25.00 & 20.22 & 27.05 & 23.51 & 18.98 \\
& Tight Frame & \underline{29.38} & \underline{26.57} & \underline{22.80} & \underline{30.02} & \underline{27.11} & \underline{23.04} & \underline{30.87} & \underline{27.99} & \underline{23.83} & \underline{28.58} & \underline{25.47} & \underline{21.65} \\
\hline
\multicolumn{2}{l|}{LR2Flow} & \textbf{34.32} & \textbf{31.70} & \textbf{28.50} & \textbf{35.35} & \textbf{32.92} & \textbf{29.88} & \textbf{35.40} & \textbf{33.15} & \textbf{30.08} & \textbf{34.86} & \textbf{32.66} & \textbf{29.65} \\ 
\Xhline{2\arrayrulewidth}
\end{tabular}
}
\end{table*}


\begin{figure*}
\centering
\caption{Qualitative comparison of denoising results on the CBSD68 benchmark at noise level $\sigma_{\rvn}=50$.}
\label{fig:qualitative_denoising}
\includegraphics[width=\linewidth]{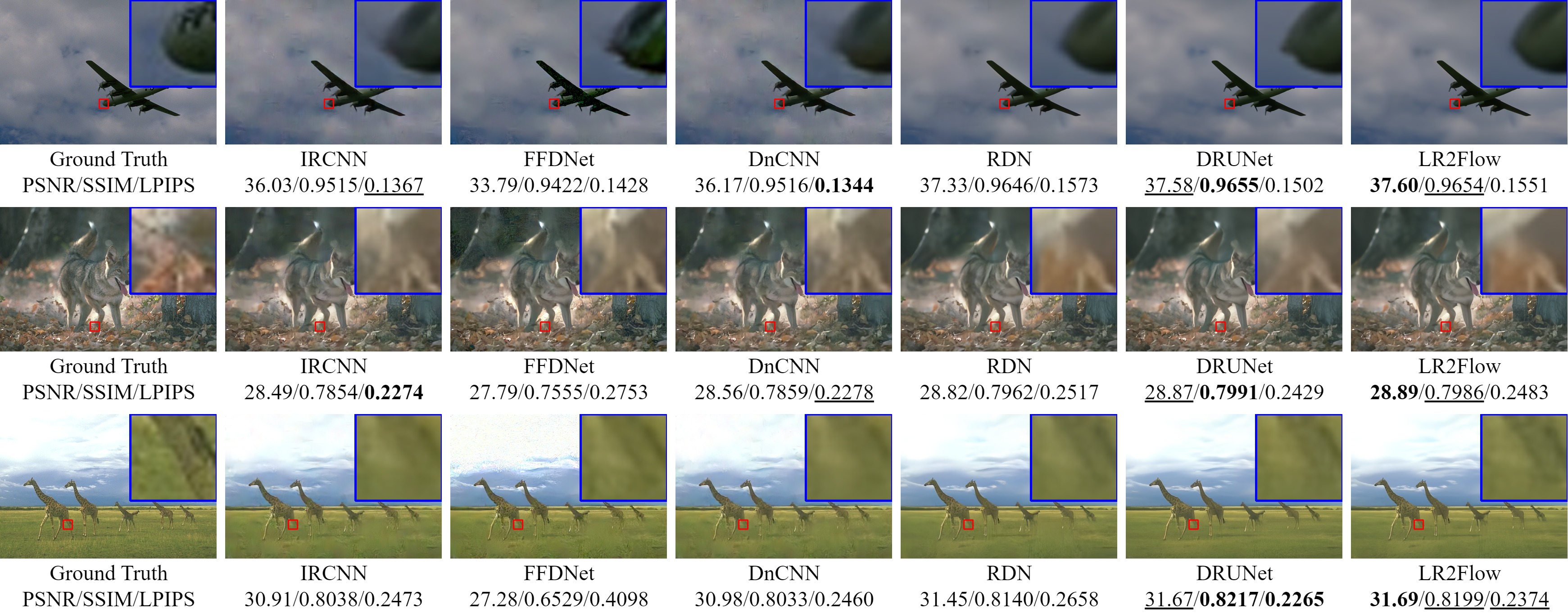}
\end{figure*}




\subsection{Ablation Study}
\label{sec:experiments_ablation}
In this section, we conduct ablation studies to assess the influence of the following factors on model performance: (i) the selection of the wavelet transformation; (ii) the choice of the invertible backbone; and (iii) the trade-off between model size and reconstruction quality.

\textbf{Comparison of $\mW$ selection.} 
We investigate the influence of the transform $\mW$ by comparing the Haar wavelet, wavelet tight frame, and PixelUnshuffle~\cite{shi2016real}. 
\cref{tab:ablation_wavelet_toy} reports the $\times2$ rescaling performance using $M=1$ flow block per level, using either the affine coupling layer~\eqref{eq:affine_coupling} or the iResBlock~\eqref{eq:iresblock} as $g^{(l)}_{i}$. 
For a fair comparison across different choices of $\mW$, we fix the number of latent channels for each invertible backbone.
Under both backbone settings, the wavelet tight frame demonstrates clear advantages in most cases, validating its theoretical advantages according to~\cref{prop:single_affine_coupling_case} and~\cref{prop:iresnet_optiom_reconstruction_error}. This advantage is also preserved in the multi-level hierarchical architecture, as shown in~\cref{tab:ablation_W_inv_arch}.

\begin{table}[ht]
\caption{Ablation study on the selection of the transform $\mW$ for the $\times2$ rescaling task, using $M=1$ flow block at each level.}
\centering
\resizebox{\linewidth}{!}{
\begin{tabular}{c|c|cc|cc|cc|cc}
\Xhline{2\arrayrulewidth}
\multirow{2}{*}{$g^{(l)}_{i}$} & \multirow{2}{*}{$\mW$} & \multicolumn{2}{c|}{Set5} & \multicolumn{2}{c|}{Set14} & \multicolumn{2}{c|}{BSD100} & \multicolumn{2}{c}{Urban100} \\
\cline{3-10}
& & PSNR & SSIM & PSNR & SSIM & PSNR & SSIM & PSNR & SSIM \\
\Xhline{2\arrayrulewidth}
\multirow{3}{*}{Affine Coupling} & PixelUnshuffle & 19.44 & 0.8163 & 18.44 & 0.8132 & 21.40 & 0.9057 & 19.33 & 0.8720 \\
& Haar & \underline{35.27} & \underline{0.9417} & \underline{31.06} & \underline{0.8890} & \underline{30.33} & \underline{0.8823} & \textbf{29.94} & \textbf{0.9235} \\
& Tight Frame & \textbf{35.55} & \textbf{0.9437} & \textbf{31.30} & \textbf{0.8910} & \textbf{30.62} & \textbf{0.8869} & \underline{29.75} & \underline{0.9203} \\
\hline
\multirow{3}{*}{iResBlock} & PixelUnshuffle & 15.32 & 0.3478 & 16.31 & 0.3118 & 17.04 & 0.2886 & 15.41 & 0.2881 \\
& Haar & \underline{15.81} & \underline{0.6147} & \underline{17.41} & \underline{0.6037} & \underline{17.77} & \underline{0.5655} & \underline{16.60} & \underline{0.5348} \\
& Tight Frame & \textbf{17.58} & \textbf{0.7092} & \textbf{19.05} & \textbf{0.6650} & \textbf{20.17} & \textbf{0.6645} & \textbf{17.91} & \textbf{0.6210} \\
\Xhline{2\arrayrulewidth}
\end{tabular}
}
\label{tab:ablation_wavelet_toy}
\end{table}

\begin{table}[ht]
\caption{Ablation study of the transform $\mW$ and invertible architectures. Experiments are conducted with $T=1$ and $M=8$. Results are reported on the DIV2K validation set for three tasks: $\times2$ image rescaling, image compression ($\times2$ rescaling with JPEG QF$=30$), and image denoising with AWGN at $\sigma_{\rvn}=25$.}
\label{tab:ablation_W_inv_arch}
\centering
\resizebox{0.8\linewidth}{!}{
\begin{tabular}{c|cccc|cc|cc}
\Xhline{2\arrayrulewidth}
\multirow{3}{*}{$\mW$} & \multicolumn{4}{c|}{Rescaling} & \multicolumn{2}{c|}{\multirow{2}{*}{Compression}} & \multicolumn{2}{c}{\multirow{2}{*}{Denoising}} \\
& \multicolumn{2}{c}{Affine Coupling} & \multicolumn{2}{c|}{iResBlock} & & & & \\
\cline{2-9}
& PSNR & SSIM & PSNR & SSIM & PSNR & SSIM & PSNR & SSIM \\
\Xhline{2\arrayrulewidth}
PixelUnshuffle & 42.34 & 0.9863 & 42.42 & 0.9869 & \underline{31.88} & \underline{0.8937} & 32.34 & 0.8918 \\
Haar & \underline{45.64} & \underline{0.9933} & 41.63 & 0.9850 & 31.68 & 0.8837 & \underline{32.63} & \underline{0.8969} \\
Tight Frame & \textbf{47.55} & \textbf{0.9958} & 43.19 & 0.9880 & \textbf{32.37} & \textbf{0.8968} & \textbf{32.84} & \textbf{0.8997} \\
\Xhline{2\arrayrulewidth}
\end{tabular}
}
\end{table}

\textbf{Comparison of invertible backbones.}
We compare the affine coupling layer \eqref{eq:affine_coupling} with the iResBlock~\eqref{eq:iresblock} for the $\times2$ rescaling task in both the single-block setting reported in~\cref{tab:ablation_wavelet_toy}, and the multi-block setting reported in~\cref{tab:ablation_W_inv_arch}.
The results show that the affine coupling layer consistently achieves superior performance
when combined with the wavelet tight frame.

\begin{table*}
\caption{Image rescaling results on the DIV2K validation set for different sizes of HCFlow~\cite{liang2021hierarchical} and LR2Flow.}
\label{tab:param_efficiency}
\centering
\resizebox{\linewidth}{!}{
\begin{tabular}{c|lc|cc|cc|cc|cc|cc}
\Xhline{2\arrayrulewidth}
\multirow{2}{*}{Scale} & \multirow{2}{*}{Model} & \multirow{2}{*}{Params} & \multicolumn{2}{c|}{Set5} & \multicolumn{2}{c|}{Set14} & \multicolumn{2}{c|}{BSD100} & \multicolumn{2}{c|}{Urban100} & \multicolumn{2}{c}{DIV2K} \\
\cline{4-13}
& & & PSNR & SSIM & PSNR & SSIM & PSNR & SSIM & PSNR & SSIM & PSNR & SSIM \\
\Xhline{2\arrayrulewidth}
\multirow{4}{*}{$\times 2$} & HCFlow & 2.07M & 45.08 & 0.9895 & 42.30 & 0.9827 & 42.61 & 0.9909 & 41.92 & 0.9928 & 45.66 & 0.9933 \\
& HCFlow-L & 4.48M & 45.33 & 0.9893 & 42.50 & 0.9828 & 43.05 & 0.9918 & 42.28 & 0.9933 & 46.01 & 0.9939 \\
& LR2Flow & 2.74M & \underline{46.87} & \underline{0.9932} & \underline{43.98} & \underline{0.9880} & \underline{44.94} & \underline{0.9948} & \underline{43.56} & \underline{0.9951} & \underline{47.55} & \underline{0.9958} \\
& LR2Flow-L & 4.95M & \textbf{47.07} & \textbf{0.9933} & \textbf{44.27} & \textbf{0.9883} & \textbf{45.29} & \textbf{0.9952} & \textbf{44.05} & \textbf{0.9954} & \textbf{47.82} & \textbf{0.9961} \\
\hline
\multirow{4}{*}{$\times 4$} & HCFlow & 4.35M & 36.29 & 0.9468 & 33.02 & 0.9065 & 31.74 & 0.8864 & 31.62 & 0.9206 & 35.23 & 0.9346 \\
& HCFlow-L & 9.38M & 36.54 & 0.9483 & 33.43 & 0.9115 & 31.97 & 0.8899 & 32.62 & 0.9307 & 35.61 & 0.9379 \\
& LR2Flow & 7.15M & \underline{37.00} & \underline{0.9493} & \underline{33.89} & \underline{0.9121} & \underline{32.33} & \underline{0.8901} & \underline{33.37} & \underline{0.9352} & \underline{35.97} & \underline{0.9383} \\
& LR2Flow-L & 9.36M & \textbf{37.06} & \textbf{0.9494} & \textbf{33.91} & \textbf{0.9122} & \textbf{32.40} & \textbf{0.8908} & \textbf{33.50} & \textbf{0.9362} & \textbf{36.04} & \textbf{0.9387} \\
\Xhline{2\arrayrulewidth}
\end{tabular}
}
\end{table*}

\textbf{Performance--complexity trade-off.}
We analyze the scalability of LR2Flow on image rescaling tasks by comparing our standard model and a higher-capacity variant, LR2Flow-L, with an enlarged state-of-the-art baseline, HCFlow-L. The results in~\cref{tab:param_efficiency} show that the standard LR2Flow outperforms HCFlow-L despite using significantly fewer parameters. Moreover, LR2Flow achieves further performance gains as model capacity increases, demonstrating the parameter efficiency and scalability of our approach.

\section{Conclusion}
We presented LR2Flow, a framework for low-resolution image representation that integrates a wavelet tight frame with normalizing flows. 
Our theoretical analysis elucidated how the data adaptivity of the wavelet tight frame influences reconstruction performance, particularly given the limited expressivity inherent in invertible architectures.
Experimental results validate our analysis.
Across image rescaling, compression, and denoising tasks, LR2Flow consistently outperforms comparable methods, demonstrating the effectiveness of our architectural design.

{
\appendix

\section{Proofs of Proposition~\ref{prop:single_affine_coupling_case}, Proposition~\ref{prop:case2_reconstruction_error}, and~\eqref{eqn:bound}}
\label{appendix:affine_coupling_beyond}

\begin{proof}[Proof of~\cref{prop:single_affine_coupling_case}]
When $\Theta$ is defined as in~\eqref{eq:affine_coupling_class}, the objective function in~\eqref{eq:reconstruction_loss_inn} can be written as
\begin{equation*}
J(\rho, \eta) = \E_{q(\rvx)p(\rvz)}\left\|\mW_H^{\top} \left((\rvz - \eta(\mW_L\rvx)) \oslash \rho(\mW_L\rvx) - \mW_H\rvx\right)\right\|^2,
\end{equation*}
where $\oslash$ denotes element-wise division. Fix $\rvx_L = \mW_L\rvx$ and $\rvx_H = \mW_H\rvx$, and let $\bm{\rho} = \rho(\rvx_L)$ and $\bm{\eta} = \eta(\rvx_L)$ for brevity. Note that the $i$-th element of the vector $\mW_H^{\top}((\rvz - \bm{\eta}) \oslash \bm{\rho} - \rvx_H)$ follows a Gaussian distribution with
\begin{equation*}
{\textstyle \text{mean=$-\sum_{j=1}^{N-d} [\mW_H^{\top}]_{ij} \left( \bm{\eta}_{j} / \bm{\rho}_{j} + [\rvx_H]_j \right)$,}}
\quad
{\textstyle \text{variance=$\sigma^2 \sum_{j=1}^{N-d} \left( [\mW_H^{\top}]_{ij} / \bm{\rho}_{j} \right)^2$.}}
\end{equation*}
Let $q^{\rvc} = \mW_{\#}q$ denote the distribution of the wavelet coefficients, and let $q^{\rvc}_{L} = [\mW_L]_{\#} q$ denote the distribution of the low-frequency coefficients. Then,
\begin{equation*}
\begin{aligned}
J(\rho,\eta) ={} & {\E_{q^{\rvc}(\rvx_L,\rvx_H)}\big[\sum_{i=1}^d \big| \sum_{j=1}^{N-d} [\mW_H^{\top}]_{ij} \left( \bm{\eta}_{j} / \bm{\rho}_{j} + [\rvx_H]_j \right) \big|^2 + \sigma^2 \sum_{i=1}^d\sum_{j=1}^{N-d} \left( [\mW_H^{\top}]_{ij} / \bm{\rho}_{j} \right)^2\big]} \\
\geq{} & \E_{q^{\rvc}(\rvx_L, \rvx_H)}\left\| \mW_H^{\top} \left( \bm{\eta} \oslash  \bm{\rho} + \rvx_H \right) \right\|^2 \geq \E_{q^{\rvc}_{L}(\rvx_L)} \Tr \left( \Var \left[\rvx_H | \rvx_L\right] \mW_H \mW_H^{\top} \right).
\end{aligned}
\end{equation*}
The minimum is attained when $\bm{\eta} \oslash \bm{\rho} + \E[\rvx_H | \rvx_L] = \mathbf{0}$.
\end{proof} 

\begin{proof}[Proof of~\cref{prop:case2_reconstruction_error}]
Let $\mP = \mI - \mW_L^{\dagger}\mW_L$.
Define $\bm{\mu}(\rvx_L) = \E[\rvx | \mW_L\rvx=\rvx_L]$. Consider a specific mapping $\gF \in \Theta$ defined in~\eqref{eq:extension_affine_coupling} with
$h(\rvx_H; \rvx_L) := -\rvx_H + \mW_H (\mW_L^{\dagger} \rvx_L - \mP \bm{\mu}(\rvx_L))$. The corresponding reconstruction is given by
\begin{equation*}
\hat{\rvx}(\rvx_L) = \E_{p(\rvz)}\left[\mW_L^{\top} \rvx_L + \mW_H^{\top} h^{-1}(\rvz; \rvx_L)\right] = \mW_L^{\top} \rvx_L + \mW_H^{\top} \mW_H (\mW_L^{\dagger} \rvx_L - \mP \bm{\mu}(\rvx_L)).
\end{equation*}
This implies $\rvx-\hat{\rvx}(\mW_L \rvx)=\mW_H^{\top}\mW_H\mP(\rvx - \bm{\mu}(\rvx_L))$. Consequently, we obtain the following bound of the reconstruction error $\rve^*$ defined in~\eqref{eq:reconstruction_loss_inn}
\begin{equation}
\label{eq:bound_extension_affine_coupling}
\begin{aligned}
\rve^{*} \leq \E_{q(\rvx)} \|\mW_H^{\top} \mW_H \mP (\rvx - \bm{\mu}(\mW_L\rvx))\|^2 = \Tr \left((\mW_H^{\top} \mW_H)^2 \mP \mSigma \mP \right).
\end{aligned}
\end{equation}
\end{proof}

\begin{proof}[Proof of~\eqref{eqn:bound}]
Let $J(\mW)$ denote the upper bound given in~\eqref{eq:bound_extension_affine_coupling}, where $\mP = \mI - \mW_L^{\dagger}\mW_L$. We have $J(\mW) \geq \sum_{i=1}^n ( \lambda^{\uparrow}_i(\mW_H^{\top} \mW_H) )^2 \cdot \lambda^{\downarrow}_i (\mP \mSigma \mP)$,
with equality when $\mV^{\uparrow}_{i}(\mW_H^{\top} \mW_H) = \mV^{\downarrow}_{i}(\mP \mSigma \mP)$ for $1\leq i\leq n$. Here $\mV^{\uparrow}_{i}(\mA)$ denotes the eigenspace of $\mA\in\sS^{n}$ corresponding to the eigenvalue $\lambda^{\uparrow}_{i}(\mA)$.
It follows that $\lambda^{\downarrow}_i (\mP \mSigma \mP) = 0$ for $n-d+1 \leq i \leq n$ and
$\lambda^{\downarrow}_i(\mP \mSigma \mP) \geq \lambda^{\downarrow}_{i+d}(\mSigma)$ for $1\leq i \leq n-d$.
Consequently,
\begin{equation*}
{\textstyle J(\mW) \geq \sum_{i=1}^{n-d}(\lambda^{\uparrow}_{i}(\mW_H^{\top}\mW_H))^2 \cdot \lambda^{\downarrow}_{i+d}(\mSigma)}.
\end{equation*}
Equality is achieved when $\text{Im}(\mP) = \bigoplus_{i=1}^{n-d}\mV^{\uparrow}_{i}(\mSigma)$, i.e., $\text{Im}(\mW_L^{T})=\bigoplus_{i=1}^{d}\mV^{\downarrow}_{i}(\mSigma)$.
\end{proof}

\section{Proof of Proposition~\ref{prop:iresnet_optiom_reconstruction_error}}
\label{appendix:proof_prop_iresnet_optim_reconstruction_error}

We first prove the following lemma.
\begin{lemma}
\label{lemma:general_inn_reconstruction}
Let $\gF \in \operatorname{Diff}(\sR^{N})$ be a diffeomorphism. We show that
\begin{equation}
\label{eq:reconstruction_error_common}
\begin{aligned}
& \E_{q(\rvx)p(\rvz)}
\|\rvx - \mW^{\top}\gF^{-1}([\gF(\mW\rvx)]_{1:d}, \rvz)\|^2 \\
\leq{} & \gC \cdot \E_{q(\rvx)} \Tr\big(\Var\big[(\gF(\mW\rvx))_{(d+1):N} | (\gF(\mW\rvx))_{1:d}\big]\big).
\end{aligned}
\end{equation}
where $\gC = \int_0^1\E_{p(\rvz)} \left\|\mW^{\top}(D_{H}\gF^{-1})\left(\bm{\xi}_{L},\bm{\xi}_{H} + t(\rvz-\bm{\xi}_H)\right)\right\|^2 \df t$.
Here, $\bm{\xi}_L \in \sR^{d}$ and $\bm{\xi}_H \in \sR^{N-d}$ are reference points, and $D_{H}$ denotes the partial Jacobian with respect to the high-frequency components.
\end{lemma}

\begin{proof}[Proof of \cref{lemma:general_inn_reconstruction}]
Given $\rvx$, let $\rvy = [\gF(\mW\rvx)]_{1:d}$ and $\rvz^{\mathrm f} = [\gF(\mW\rvx)]_{(d+1):N}$ denote the transformed low- and high-frequency components, respectively.
Let $\hat{\rvx} = \E_{p(\rvz)}[\mW^{\top}\gF^{-1}(\rvy,\rvz)]$ be the reconstruction. Then
the reconstruction error writes
\begin{equation*}
\rvx - \hat{\rvx} = \E_{p(\rvz)} \big[{\textstyle\int_{0}^{1}} \left(\mW^{\top}(D_{H} \gF^{-1}) \left(\rvy, \rvz + t(\rvz^{\text{f}} - \rvz)\right)\right)^{\top} (\rvz^{\text{f}} - \rvz) \df t\big].
\end{equation*}
Let $g(t,\rvz) = \mW^{\top}(D_{H}\gF^{-1})(\rvy, \rvz + t(\rvz^{\mathrm f} - \rvz))$.
Furthermore, denote
$q^{\text{f}} = (\gF\circ\mW)_{\#}q$ as the model-induced joint coefficient distribution
and $q^{\text{f}}_{Z} = \big([\gF\circ\mW]_{(d+1):N}\big)_{\#} q$ as the model-induced latent distribution. We show that
\begin{equation*}
\begin{aligned}
\E_{q(\rvx)}\|\rvx - \hat{\rvx}\|^2 
\leq{} & 
\E_{q^{\text{f}}(\rvy,\rvz^{\text{f}})} \big\|\E_{p(\rvz)}{\textstyle\int_{0}^{1}} g(t,\rvz)^{\top} (\rvz - \rvz^{\text{f}}) \df t\big\|^2 \\
\leq{} & \E_{q^{\text{f}}(\rvy,\rvz^{\text{f}})} \big[\big(\E_{p(\rvz)} {\textstyle\int_{0}^{1}} \|g(t,\rvz)\|^2 \df t\big) \cdot \big(\E_{p(\rvz)}\|\rvz - \rvz^{\text{f}}\|^2\big)\big] \\
={} & \gC \cdot \E_{q^{\text{f}}_{Z}(\rvz^{\text{f}}) p(\rvz)}\|\rvz^{\text{f}} - \rvz\|^2,
\end{aligned}
\end{equation*}
where $\gC = \int_0^1\E_{p(\rvz)} \left\|\mW^{\top}D_{H}\gF^{-1}\left(\bm{\xi}_{L},\bm{\xi}_{H} + t(\rvz-\bm{\xi}_H)\right)\right\|^2 \df t$ for some reference points $\bm{\xi}_L \in \sR^{d}$ and $\bm{\xi}_H \in \sR^{N-d}$. In addition,
\begin{equation*}
\E_{q^{\text{f}}_{Z}(\rvz^{\text{f}}) p(\rvz)}\|\rvz^{\text{f}} - \rvz\|^2 = \E_{q^{\text{f}}_{Z}(\rvz^{\text{f}})}\|\rvz^{\text{f}}\|^2 + (N-d)\sigma^2 \geq \E\big[\Tr(\Var[\rvz^{\text{f}} | \rvy])\big],
\end{equation*}
with equality holding when $\E[\rvz^{\text{f}} | \rvy] = \mathbf{0}$.
\end{proof}

\begin{proof}[Proof of~\cref{prop:iresnet_optiom_reconstruction_error}]
We focus on the subclass of $\Theta$ defined in~\eqref{eq:iresblock}, specifically considering mappings of the form
\begin{equation}
\label{eq:specific_iresnet}
\begin{aligned}
& \gF:(\rvx_L, \rvx_{H}) \mapsto (\rvx_{L} + f(\rvx_{L}), \rvx_{H} + h(\rvx_{H}; \rvx_{L})), \\
& \text{$f:\sR^{d}\to\sR^{d}$, Lip$(f)\leq L$; $h(\cdot; \rvx_L):\sR^{N-d}\to\sR^{N-d}$, $\text{Lip}\big(h(\cdot; \rvx_{L})\big) \leq L$.}
\end{aligned}
\end{equation}
Let $q^{\rvc} = \mW_{\#}q$ denote the distribution of the wavelet coefficients. According to~\eqref{eq:reconstruction_error_common}, the reconstruction error in~\eqref{eq:reconstruction_loss_inn} is bounded by 
\begin{equation*}
\rve^* \leq \frac{\inf_{\text{$f,h$ satisfy~\eqref{eq:specific_iresnet}}}\E_{q^{\rvc}(\rvx_L, \rvx_H)}\big[\Tr(\Var[\rvx_{H} + h(\rvx_{H};\rvx_{L}) | \rvx_L])\big]}{(1-L)^2},
\end{equation*}
where we use the following properties: (i) Lip$(\gF^{-1}) \leq (1-L)^{-1}$; and (ii) the $\sigma$-algebra generated by $\rvx_L$ coincides with that generated by $\rvx_{L} + f(\rvx_{L})$. Furthermore, for the specific case where $h(\rvx_H; \rvx_L)=\rho(\rvx_L)\odot \rvx_H+\eta(\rvx_L)$ with $\text{Lip}(\rho)\leq L$, we show that
\begin{equation*}
\begin{aligned}
\Tr\big( \Var[\rvx_H + h(\rvx_H;\rvx_L)\mid \rvx_L] \big)
={} & \Tr\big( \text{diag}(\left(\vone_{d} + \rho(\rvx_L)\right)^2) \ \Var[\rvx_H | \rvx_L] \big) \\
\leq{} & \|\left(\vone_{d} + \rho(\rvx_L)\right)^2\| \cdot \|\Var[\rvx_H | \rvx_L]\|_{F},
\end{aligned}
\end{equation*}
which yields the bound established in~\eqref{eq:iresnet_bound2}.
\end{proof}

\section{Reconstruction Error Analysis for Orthogonal Transformation}
\label{appendix:proof_linear_orthogonal}
In this case, the hypothesis space is given by
\begin{equation}
\label{eq:orthogonal_class}
\Theta = O(N) = \{\mF\in\sR^{N\times N} \mid \mF\mF^{\top} = \mF^{\top}\mF = \mI\}.
\end{equation}
By focusing on this set, we aim to investigate the inherent limitations of linear models. We derive the following analytic form for the reconstruction error $\rve^*$ defined in~\eqref{eq:reconstruction_loss_inn}.

\begin{theorem}
\label{thm:inn_linear_orthogonal}
Let $\Theta$ be defined as in~\eqref{eq:orthogonal_class}. Then the minimal reconstruction error defined as in~\eqref{eq:reconstruction_loss_inn} is given by
\begin{equation}
\label{eq:optim_inn_linear}
{\textstyle \rve^* = \sum_{i=d+1}^{n}\lambda^{\downarrow}_{i}(\mSigma) + (n-d)\sigma^2.}
\end{equation}
\end{theorem}

To proof~\cref{thm:inn_linear_orthogonal}, we utilize the following lemma.

\begin{lemma}
\label{lemma:linear_case_lemma}
Let $\mW \in \sR^{N\times n}$ satisfies $\mW^{\top}\mW = \mI$, and $\mS\in\sS^{n}_{++}$.
Then, the optimal value of the minimization problem
\begin{equation}
\label{eq:objective_linear_case}
\min_{\mA \in \sR^{(N-d)\times N},\mA\mA^{\top} = \mI} \Tr\left((\mW^{\top}\mA^{\top}\mA\mW)^2\mS\right) + \sigma^2 \Tr(\mW^{\top}\mA^{\top}\mA\mW)
\end{equation}
is $\sum_{i=d+1}^{n}\lambda^{\downarrow}_{i}(\mS) + (n-d)\sigma^2$.
\end{lemma}

\begin{proof}[Proof of~\cref{lemma:linear_case_lemma}]
For $\mA \in \sR^{(N-d)\times N}$ satisfies $\mA\mA^{\top} = \mI$, let $\mP = \mA^{\top}\mA$. Let $\mX = \mW \mS\mW^{\top}$ and $\mY = \mW\mW^{\top}$.
The objective in~\eqref{eq:objective_linear_case} can be rewritten as
\begin{equation*}
\begin{aligned}
J(\mP) = \Tr(\mP\mY\mP\mX) + \sigma^2\Tr(\mP\mY).
\end{aligned}
\end{equation*}
Note that: (i) $\mX$ and $\mY$ commute: $\mX\mY = \mY\mX = \mX$, implying they can be simultaneously diagonalized; (ii) $\text{Im}(\mX) = \text{Im}(\mY) = \text{Im}(\mW)$. Consequently, there exists $\mQ\in O(N)$ such that $\mX = \mQ \Lambda_{X}\mQ^{\top}$ and $\mY = \mQ \Lambda_{Y}\mQ^{\top}$, where
\begin{equation*}
\Lambda_{X} = \text{diag}(\lambda^{\downarrow}_{1}(\mS), \ldots, \lambda^{\downarrow}_{n}(\mS), \mathbf{0}_{N-n}),
\quad 
\Lambda_{Y} = \text{diag}(\mathbf{1}_{n}, \mathbf{0}_{N-n}),
\end{equation*}
where $\mathbf{1}_{n}\in\sR^{n}$ denotes a vector of all ones, and $\mathbf{0}_{N-d}\in\sR^{N-n}$ denotes a vector of all zeros.
Let $\mU = \mQ^{\top}\mP \mQ$,
the objective then writes $J(\mP) = \Tr(\mU \Lambda_{X} \mU \Lambda_{Y}) + \sigma^2 \Tr(\mU \Lambda_{Y})$.
We invoke the following results: (i) [Von Neumann's trace inequality] $\Tr(\mB\mC) \geq \sum_{i=1}^{m} \lambda^{\downarrow}_{i}(\mB)\lambda^{\downarrow}_{m-i}(\mC)$ for Hermitian matrices $\mB,\mC \in \sR^{m\times m}$; and (ii) [Cauchy Interlacing Theorem] For any column orthogonal $\mV \in \sR^{m\times k}$ and Hermitian $\mB \in \sR^{m\times m}$, $\lambda^{\downarrow}_{i}(\mB) \geq \lambda^{\downarrow}_{i}(\mV^{\top}\mB \mV) \geq \lambda^{\downarrow}_{i+m-k}(\mB)$, $i=1,\ldots, k$. Consequently,
\begin{equation*}
\begin{aligned}
\Tr(\mU \Lambda_{X} \mU \Lambda_{Y}) \geq{} & {\textstyle \sum_{i=1}^{N} \lambda^{\downarrow}_{i}(\mU \Lambda_{X}\mU)\lambda^{\downarrow}_{N-i}(\Lambda_{Y}) \geq{} \sum_{i=1}^{n} \lambda^{\downarrow}_{i}(\mU \Lambda_{X}\mU) \geq{} \sum_{i=d+1}^{n}\lambda^{\downarrow}_{i}(\mS),} \\
\Tr(\mU \Lambda_{Y}) \geq{} & {\textstyle \sum_{i=1}^{N} \lambda^{\downarrow}_{i}(\mU\mU)\lambda^{\downarrow}_{N-i}(\Lambda_{Y}) ={} n-d.}
\end{aligned}
\end{equation*}
The minimum is achieved when $\mP = \mQ\mQ^{\top}$.
\end{proof}

\begin{proof}[Proof of~\cref{thm:inn_linear_orthogonal}]
Let $\mF = (\mF_{1}^{\top}, \mF_{2}^{\top})^{\top}\in O(N)$, where $\mF_{1}\in\sR^{d\times N}$ and $\mF_{2}\in\sR^{(N-d)\times N}$. The reconstruction of $\rvx$ is given by
\begin{equation*}
\hat{\rvx} = \mW^{\top}\mF^{\top}([\mF(\mW\rvx)]_{1:d}, \rvz)
= \mW^{\top}\mF_{1}^{\top}\mF_{1}\mW\rvx + \mW^{\top}\mF_{2}^{\top}\rvz.
\end{equation*}
Since $\mW^{\top}(\mF^{\top}_{1}\mF_{1} + \mF^{\top}_{2}\mF_{2})\mW = \mI$,
the reconstruction error can be written as
\begin{equation*}
\begin{aligned}
J(\mF) 
={} & \Tr\left((\mW^{\top}\mF_2^{\top}\mF_{2}\mW)^2\mSigma\right) + \sigma^2 \Tr(\mW^{\top}\mF_2^{\top}\mF_{2}\mW).
\end{aligned}
\end{equation*}
According to~\cref{lemma:linear_case_lemma}, $\rve^* = \sum_{i=d+1}^{n}\lambda^{\downarrow}_{i}(\mSigma) + (n-d)\sigma^2$. 
\end{proof}

As demonstrated in~\cref{thm:inn_linear_orthogonal}, for the linear hypothesis space defined in~\eqref{eq:orthogonal_class}, the reconstruction error $\rve^*$ remains invariant to the choice of $\mW$, as the optimal orthogonal mapping naturally aligns the low-frequency subband with the principal components of the data. 
Nevertheless, linear mappings possess limited expressivity and fail to approximate the optimal solution to \eqref{eq:reconstruction_loss_inn} for general data distributions $q(\rvx)$. In the following example, we show that a nonlinear invertible mapping $\gF$ yields a strictly lower reconstruction error than the linear bound derived in~\eqref{eq:optim_inn_linear}.
This observation further supports the architectural design of LR2Flow, which hierarchically composes affine coupling layers and iResBlocks, both of which have been shown to be universal approximators of homeomorphisms~\cite{teshima2020coupling,zhang2020approximation}.

\begin{example}
Consider a two-dimensional zero-mean Gaussian with $\mSigma=\tau^2\mI$.  For a LR representation dimension of $d=1$, the optimal reconstruction error $\rve^*$ in~\eqref{eq:reconstruction_loss_inn} is strictly less than the bound $\lambda^{\downarrow}_2(\mSigma) = \tau^2$ provided in~\eqref{eq:optim_inn_linear}.
\end{example}

\begin{proof}
Let $(x, y) \sim \gN(0, \tau^2 \mI)$. Consider the polar coordinate transformation $x = \tau r\cos\theta$ and $y = \tau r\sin\theta$, with $r\geq 0$ and $\theta\in[0,2\pi)$. The joint density is given by $q(r, \theta) = \frac{r}{2\pi} \mathrm{e}^{-\frac{r^2}{2}}$, which implies that $r$ and $\theta$ are independent with marginals $p(r)=r \mathrm{e}^{-\frac{r^2}{2}}$ and $p(\theta)=\frac{1}{2\pi}\sI(0\leq\theta < 2\pi)$.
Consequently,
\begin{equation*}
\begin{aligned}
\E [(x - \E[x | \theta])^2 + (y - \E[y | \theta])^2]
={} & \E \left[ \tau^2 (r - \E[r])^2 \left(\cos^2\theta + \sin^2\theta\right) \right] 
={} \tau^2 \Var[r].
\end{aligned}
\end{equation*}
Direct calculation yields $\E[r] = \sqrt{\frac{\pi}{2}}$ and $\Var[r]=2-\frac{\pi}{2}$. We define the nonlinear invertible map as $\gF = f\circ \mW^{\top}$, where $f(x,y) = (\theta(x,y), \tau^{-1}\sqrt{x^2+y^2})$ and $\theta(x,y)$ denotes the angle. By choosing the latent prior $p(z)=\delta(z-\E[r])$, we obtain
\begin{equation*}
\E_{q(\rvx)p(z)}\left\|\rvx - \mW^{\top}\gF^{-1}\left(\left[\gF\left(\mW\rvx\right)\right]_{1:d}, z\right)\right\|^2 = (2 - \pi/2)\tau^2.
\end{equation*}
Thus, $\rve^* \leq (2 - \frac{\pi}{2})\tau^2 < \tau^2 = \lambda^{\downarrow}_{2}(\mSigma)$.
\end{proof}

\section{Two-Dimensional Wavelet Tight Frame Construction}
\label{sec:2d_construction}
In this section, we present the construction of the two-dimensional wavelet tight frame employed in our image experiments.
The one-dimensional wavelet tight frame generated by linear B-splines~\cite{cai2012image,dong2010mra} consists of one low-pass filter
$\rvk_{0} = (\frac{1}{4}, \frac{1}{2}, \frac{1}{4})$,
and two high-pass filters
$\rvk_{1} = (-\frac{1}{4}, \frac{1}{2}, -\frac{1}{4})$ and $\rvk_{2} = (\frac{\sqrt{2}}{4}, 0, -\frac{\sqrt{2}}{4})$.
The two-dimensional filters are constructed via tensor products
\begin{equation*}
{\textstyle \rvk_{(i,j)}=\rvk_i^\top \rvk_j, \quad
0\leq i,j\leq 2.}
\end{equation*}
Thus, the two-dimensional filter bank contains one low-pass filter
$\rvk^{\mathrm{lp}} = \rvk_{(0,0)}$
and $r=8$ high-pass filters $\rvk^{\mathrm{hp}}_{(i,j)} = \rvk_{(i,j)}$, $(i,j) \in \gI_{H} := \{(i,j) \mid 0 \leq i,j \leq 2, \, (i,j) \neq (0,0) \}$.

We denote the two-dimensional downsampling operator by $\downarrow_{2,2}$:
\begin{equation*}
{\textstyle \downarrow_{2,2}: \sR^{H\times W} \to \sR^{\frac{H}{2}\times \frac{W}{2}}, \quad 
(\downarrow_{2,2}\rvx)[i,j] = \rvx[2i, 2j],}
\end{equation*}
and the corresponding zero-insertion upsampling operator by $\uparrow_{2,2}$:
\begin{equation*}
\begin{aligned}
\uparrow_{2,2}: \sR^{h\times w}\to
\sR^{(2h)\times (2w)}, \quad
(\uparrow_{2,2}\rvx)[i,j]
=
\begin{cases}
\rvx[\tfrac{i}{2},\tfrac{j}{2}],
& \text{if } i,j \text{ are both even},\\
0,& \text{otherwise}.
\end{cases}
\end{aligned}
\end{equation*}

For a single-channel image $\rvx\in\sR^{H\times W}$, the wavelet analysis operator is denoted by
$\mW = (\mW_L^{\top}, \mW_H^{\top})^{\top}$, where the low- and high-frequency components are defined by
\begin{equation*}
\begin{aligned}
& {\textstyle \mW_L: \rvx \mapsto \rvx_{L} := \downarrow_{2,2} (\rvk^{\mathrm{lp}}\ast \rvx) \in \sR^{\frac{H}{2}\times \frac{W}{2}},} \\
& {\textstyle \mW_H: \rvx \mapsto \{\rvx_{H, (i,j)}\}_{(i,j) \in \gI_{H}}, \quad
\text{where } \rvx_{H,(i,j)} := \downarrow_{2,2} (\rvk_{(i,j)}^{\mathrm{hp}}\ast \rvx) \in \sR^{\frac{H}{2}\times \frac{W}{2}}.}
\end{aligned}
\end{equation*}
Thus, for each input channel, the transform produces $9$ subbands, each downsampled by a factor of $2$ along each spatial direction. This corresponds to a redundancy factor of $9/4$ relative to the original spatial resolution.

The corresponding synthesis operator is $\mW^\top=(\mW_L^\top,\mW_H^\top)$, where the low- and high-frequency synthesis operators are defined by
\begin{equation*}
\begin{aligned}
& {\textstyle \mW_L^{\top}: \rvx_{L} \mapsto \tilde{\rvk}^{\mathrm{lp}} \ast (\uparrow_{2,2} \rvx_{L}) \in \sR^{H\times W},} \\
& {\textstyle \mW_{H}^{\top}: \{\rvx_{H, (i,j)}\}_{(i,j) \in \gI_{H}} \mapsto \sum_{(i,j) \in \gI_{H}} \tilde{\rvk}^{\mathrm{hp}}_{(i,j)} \ast (\uparrow_{2,2} \rvx_{H, (i,j)}) \in \sR^{H\times W},}
\end{aligned}
\end{equation*}
where $\tilde{\rvk}^{\mathrm{lp}} = \rvk^{\mathrm{lp}}[-\cdot, -\cdot]$ and $\tilde{\rvk}^{\mathrm{hp}}_{(i,j)} = \rvk^{\mathrm{hp}}_{(i,j)}[-\cdot, -\cdot]$ denote the spatially flipped reconstruction filters.
The operators $\mW^{\top}$, $\mW_L^{\top}$, and $\mW_{H}^{\top}$ are the adjoints of their corresponding analysis operators.

\textbf{Implementation.}
For color images $\rvx\in\sR^{H\times W \times C}$, the two-dimensional wavelet transform is applied independently to each color channel. Equivalently, the implementation uses grouped convolutions with fixed wavelet filters. For images with $C$ channels, the wavelet transform produces $9C$ coefficient channels at spatial resolution $\frac{H}{2}\times\frac{W}{2}$, consisting of $C$ low-frequency channels and $8C$ high-frequency channels.

We use reflection-based mirror extension for boundary handling. In the analysis transform, the input is padded by one pixel along each spatial direction using reflection padding before convolution with the $3\times 3$ tensor-product filters. In the synthesis transform, the coefficient maps are first upsampled by zero insertion and then padded using the same mirror extension before convolution with the flipped reconstruction filters. For subbands associated with antisymmetric one-dimensional high-pass filters, the reflected boundary values are additionally multiplied by $-1$ along the corresponding spatial direction.
This sign correction ensures that the mirror extension is consistent with the parity of the tensor-product wavelet filters.

\section{Network Implementation}
\label{sec:net_impl}
In this section, we describe the implementation of the LR2Flow network. Recall that LR2Flow adopts a $T$-level hierarchical architecture, where each level consists of a wavelet operator $\mW$ and a flow mapping $f^{(l)}$. Each flow mapping consists of $M$ flow blocks,
$f^{(l)} = f^{(l)}_{M} \circ \cdots \circ f^{(l)}_{1}$.
Each flow block consists of an ActNorm layer, an invertible $1\times 1$ convolution, and a nonlinear invertible transformation $g^{(l)}_{i}$:
\begin{equation*}
{\textstyle f^{(l)}_{i} = g^{(l)}_{i} \circ \mathrm{Inv}_{1\times1} \circ \mathrm{ActNorm}.}
\end{equation*}
Following~\cite{kingma2018glow}, we implement the ActNorm layer and invertible $1\times1$ convolution, and instantiate $g^{(l)}_{i}$ as the affine coupling layer
\begin{equation*}
{\textstyle g^{(l)}_{i}(\rvc_{L}, \rvc_{H}) = 
\big(\rvc_{L}, 
\rvc_{H} \odot \exp(\rho^{(l)}_{i}(\rvc_{L})) + \eta^{(l)}_{i}(\rvc_{L})\big),}
\end{equation*}
where $\rho^{(l)}_{i}$ and $\eta^{(l)}_{i}$ are parameterized by DenseBlocks~\cite{wang2018esrgan,huang2017densely}. The number of hidden channels is set to $64$ for image rescaling, $96$ for image compression, and $128$ for image denoising. For image denoising, we introduce an additional high-frequency restoration network $\gR$, implemented as a Residual Dense Network (RDN)~\cite{zhang2021residual} with hidden channel width $\mathrm{nf}=128$, growth channel width $\mathrm{gc}=64$, and $\mathrm{nb}=3$ Residual-in-Residual Dense Blocks (RRDBs)~\cite{wang2018esrgan}.
This restoration module contains $8.83$M parameters.

\section{Ablation of the High-Frequency Restoration Network}
In this section, we investigate the effect of the high-frequency restoration network $\gR$ in the denoising model. Simply removing $\gR$ while retaining the original forward--reverse procedure would reduce the model to the trivial identity mapping
\begin{equation*}
{\textstyle \hat{\rvx} = (\mW^{\top} \circ \gF^{-1}) \circ (\gF \circ \mW) (\rvx_{\rvn}) = \rvx_{\rvn},}
\end{equation*}
which merely reproduces the noisy input. We therefore compare our method with the following denoising pipeline~\cite{liu2021invertible}:
\begin{equation*}
{\textstyle (\rvy_{\rvn}, \rvz_{\rvn}) = (\gF\circ \mW)(\rvx_{\rvn}), \quad
\hat{\rvx} = \E_{p(\rvz)}[(\mW^{\top} \circ \gF^{-1})(\rvy_{\rvn}, \rvz)],}
\end{equation*}
where the latent variable $\rvz_{\rvn}$ obtained during the forward pass is replaced by latent samples $\rvz\sim p$ in the reverse process. 
Specifically, we use $p = \gN(\mathbf{0}, \mI)$ during training and set $\rvz \equiv \mathbf{0}$ during inference. 
We train this variant using
\begin{equation*}
{\textstyle \gL_{\text{denoising}}' = \lambda_{\text{img}}\gL_{\text{img}} + \lambda_{\text{lf}}\gL_{\text{lf}},}
\end{equation*}
where $\gL_{\text{img}} = \|\hat{\rvx} - \rvx_{\rvc}\|_{1}$ and $\gL_{\text{lf}} = \|\rvy_{\rvn} - \rvy_{\rvc}\|_2^2$ are the image reconstruction loss and low-frequency alignment loss defined in the main text, respectively. We set $\lambda_{\text{img}} = 1.0$ and $\lambda_{\text{lf}} = 10^{-2}$. As shown in~\cref{tab:ablate_R}, incorporating $\gR$ substantially improves denoising performance. We attribute this improvement to the clearer and more informative high-frequency latent representation recovered by $\gR$, which helps preserve structural information.

\begin{table*}
\centering
\caption{Ablation study of the high-frequency restoration network $\gR$. PSNR results are reported.}
\label{tab:ablate_R}
\resizebox{\linewidth}{!}
{
\begin{tabular}{c|ccc|ccc|ccc|ccc}
\Xhline{2\arrayrulewidth}
\multirow{2}{*}{$\gR$} & \multicolumn{3}{c|}{CBSD68} & \multicolumn{3}{c|}{Kodak24} & \multicolumn{3}{c|}{McMaster} & \multicolumn{3}{c}{Urban100} \\
\cline{2-13}
& $\sigma_{\rvn}$=15 & $\sigma_{\rvn}$=25 & $\sigma_{\rvn}$=50 & $\sigma_{\rvn}$=15 & $\sigma_{\rvn}$=25 & $\sigma_{\rvn}$=50 & $\sigma_{\rvn}$=15 & $\sigma_{\rvn}$=25 & $\sigma_{\rvn}$=50 & $\sigma_{\rvn}$=15 & $\sigma_{\rvn}$=25 & $\sigma_{\rvn}$=50 \\
\Xhline{2\arrayrulewidth}
\cmark & \textbf{34.32} & \textbf{31.70} & \textbf{28.50} & \textbf{35.35} & \textbf{32.92} & \textbf{29.88} & \textbf{35.40} & \textbf{33.15} & \textbf{30.08} & \textbf{34.86} & \textbf{32.66} & \textbf{29.65} \\
\xmark & 34.22 & 31.60 & 28.41 & 35.21 & 32.79 & 29.75 & 35.17 & 32.97 & 29.93 & 34.52 & 32.35 & 29.33 \\
\Xhline{2\arrayrulewidth}
\end{tabular}
}
\end{table*}

\section{Ablation Study of R-D Balancing}
\label{sec:r-d_balance}
Recall that in the image compression application, the R-D balancing variant of LR2Flow, denoted by LR2Flow(R), is trained with the objective
\begin{equation*}
{\textstyle\gL_{\text{comp}} = \gL + \lambda_{\text{rate}} \gL_{\text{rate}},}
\end{equation*}
where $\lambda_{\text{rate}}$ controls the strength of the rate regularization.
In this section, we ablate different choices of $\lambda_{\text{rate}}$. As shown in~\cref{fig:ablate_R-D}, $\lambda_{\text{rate}} = 3\times 10^{-3}$ achieves the best overall R-D trade-off. Decreasing $\lambda_{\text{rate}}$ does not significantly change the R-D curve, whereas increasing $\lambda_{\text{rate}}$ leads to a performance drop. We therefore set $\lambda_{\text{rate}} = 3\times 10^{-3}$ for training LR2Flow(R).

\begin{figure}
\centering
\includegraphics[width=\linewidth]{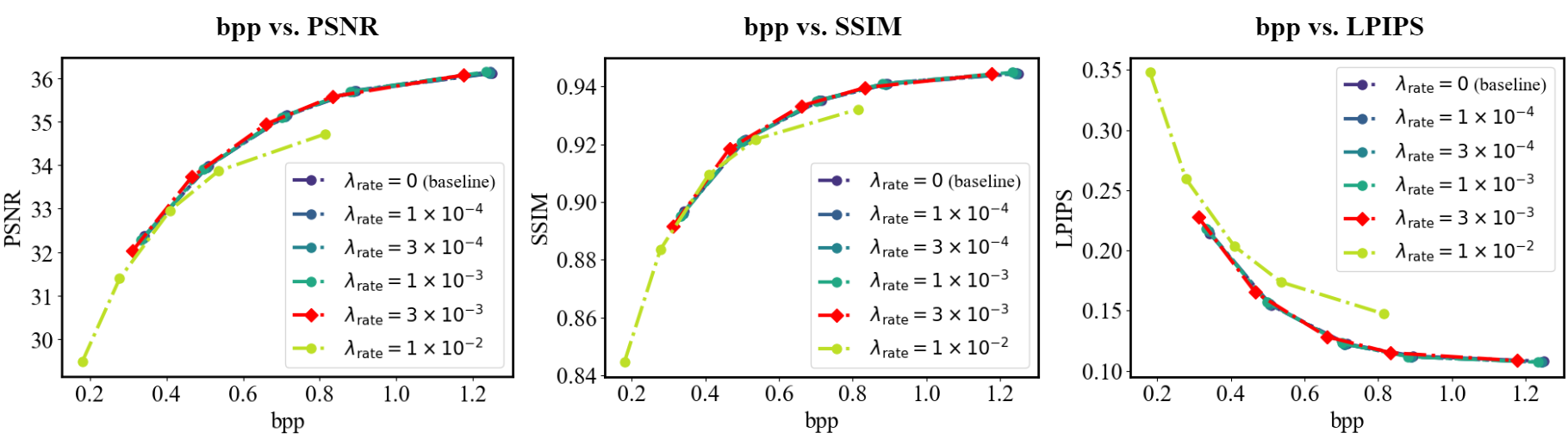}
\caption{Ablation study on the effect of $\lambda_{\text{rate}}$ on the R-D trade-off in image compression.}
\label{fig:ablate_R-D}
\end{figure}

}

\bibliographystyle{siamplain}
\bibliography{references}

@article{dinh2016density,
  title={Density estimation using real nvp},
  author={Dinh, Laurent and Sohl-Dickstein, Jascha and Bengio, Samy},
  journal={arXiv preprint arXiv:1605.08803},
  year={2016}
}

@article{kingma2018glow,
  title={Glow: Generative flow with invertible 1x1 convolutions},
  author={Kingma, Durk P and Dhariwal, Prafulla},
  journal={Advances in neural information processing systems},
  volume={31},
  year={2018}
}

@inproceedings{yang2023self,
  title={Self-asymmetric invertible network for compression-aware image rescaling},
  author={Yang, Jinhai and Guo, Mengxi and Zhao, Shijie and Li, Junlin and Zhang, Li},
  booktitle={Proceedings of the AAAI Conference on Artificial Intelligence},
  volume={37},
  pages={3155--3163},
  year={2023}
}

@article{ron1997affine,
  title={Affine systems inL2 (Rd): the analysis of the analysis operator},
  author={Ron, Amos and Shen, Zuowei},
  journal={Journal of Functional Analysis},
  volume={148},
  number={2},
  pages={408--447},
  year={1997},
  publisher={Elsevier}
}

@article{daubechies2003framelets,
  title={Framelets: MRA-based constructions of wavelet frames},
  author={Daubechies, Ingrid and Han, Bin and Ron, Amos and Shen, Zuowei},
  journal={Applied and computational harmonic analysis},
  volume={14},
  number={1},
  pages={1--46},
  year={2003},
  publisher={Elsevier}
}

@article{dong2010mra,
  title={MRA-based wavelet frames and applications},
  author={Dong, Bin and Shen, Zuowei and others},
  journal={IAS Lecture Notes Series, Summer Program on “The Mathematics of Image Processing”, Park City Mathematics Institute},
  volume={19},
  pages={9--158},
  year={2010},
  publisher={Citeseer}
}

@article{cai2012image,
  title={Image restoration: total variation, wavelet frames, and beyond},
  author={Cai, Jian-Feng and Dong, Bin and Osher, Stanley and Shen, Zuowei},
  journal={Journal of the American Mathematical Society},
  volume={25},
  number={4},
  pages={1033--1089},
  year={2012}
}

@article{xiao2023invertible,
  title={Invertible rescaling network and its extensions},
  author={Xiao, Mingqing and Zheng, Shuxin and Liu, Chang and Lin, Zhouchen and Liu, Tie-Yan},
  journal={International Journal of Computer Vision},
  volume={131},
  number={1},
  pages={134--159},
  year={2023},
  publisher={Springer}
}

@article{xing2023scale,
  title={Scale-arbitrary invertible image downscaling},
  author={Xing, Jinbo and Hu, Wenbo and Xia, Menghan and Wong, Tien-Tsin},
  journal={IEEE Transactions on Image Processing},
  year={2023},
  publisher={IEEE}
}

@inproceedings{liu2021invertible,
  title={Invertible denoising network: A light solution for real noise removal},
  author={Liu, Yang and Qin, Zhenyue and Anwar, Saeed and Ji, Pan and Kim, Dongwoo and Caldwell, Sabrina and Gedeon, Tom},
  booktitle={Proceedings of the IEEE/CVF conference on computer vision and pattern recognition},
  pages={13365--13374},
  year={2021}
}

@article{sun2020learned,
  title={Learned image downscaling for upscaling using content adaptive resampler},
  author={Sun, Wanjie and Chen, Zhenzhong},
  journal={IEEE Transactions on Image Processing},
  volume={29},
  pages={4027--4040},
  year={2020},
  publisher={IEEE}
}

@inproceedings{kim2018task,
  title={Task-aware image downscaling},
  author={Kim, Heewon and Choi, Myungsub and Lim, Bee and Lee, Kyoung Mu},
  booktitle={Proceedings of the European conference on computer vision (ECCV)},
  pages={399--414},
  year={2018}
}

@article{dong2015image,
  title={Image super-resolution using deep convolutional networks},
  author={Dong, Chao and Loy, Chen Change and He, Kaiming and Tang, Xiaoou},
  journal={IEEE transactions on pattern analysis and machine intelligence},
  volume={38},
  number={2},
  pages={295--307},
  year={2015},
  publisher={IEEE}
}

@inproceedings{lim2017enhanced,
  title={Enhanced deep residual networks for single image super-resolution},
  author={Lim, Bee and Son, Sanghyun and Kim, Heewon and Nah, Seungjun and Mu Lee, Kyoung},
  booktitle={Proceedings of the IEEE conference on computer vision and pattern recognition workshops},
  pages={136--144},
  year={2017}
}

@inproceedings{liang2021swinir,
  title={Swinir: Image restoration using swin transformer},
  author={Liang, Jingyun and Cao, Jiezhang and Sun, Guolei and Zhang, Kai and Van Gool, Luc and Timofte, Radu},
  booktitle={Proceedings of the IEEE/CVF international conference on computer vision},
  pages={1833--1844},
  year={2021}
}

@article{teshima2020coupling,
  title={Coupling-based invertible neural networks are universal diffeomorphism approximators},
  author={Teshima, Takeshi and Ishikawa, Isao and Tojo, Koichi and Oono, Kenta and Ikeda, Masahiro and Sugiyama, Masashi},
  journal={Advances in Neural Information Processing Systems},
  volume={33},
  pages={3362--3373},
  year={2020}
}

@inproceedings{xiao2020invertible,
  title={Invertible image rescaling},
  author={Xiao, Mingqing and Zheng, Shuxin and Liu, Chang and Wang, Yaolong and He, Di and Ke, Guolin and Bian, Jiang and Lin, Zhouchen and Liu, Tie-Yan},
  booktitle={Computer Vision--ECCV 2020: 16th European Conference, Glasgow, UK, August 23--28, 2020, Proceedings, Part I 16},
  pages={126--144},
  year={2020},
  organization={Springer}
}

@inproceedings{liang2021hierarchical,
  title={Hierarchical conditional flow: A unified framework for image super-resolution and image rescaling},
  author={Liang, Jingyun and Lugmayr, Andreas and Zhang, Kai and Danelljan, Martin and Van Gool, Luc and Timofte, Radu},
  booktitle={Proceedings of the IEEE/CVF International Conference on Computer Vision},
  pages={4076--4085},
  year={2021}
}

@article{quan2015data,
  title={Data-driven multi-scale non-local wavelet frame construction and image recovery},
  author={Quan, Yuhui and Ji, Hui and Shen, Zuowei},
  journal={Journal of Scientific Computing},
  volume={63},
  pages={307--329},
  year={2015},
  publisher={Springer}
}

@article{cai2014data,
  title={Data-driven tight frame construction and image denoising},
  author={Cai, Jian-Feng and Ji, Hui and Shen, Zuowei and Ye, Gui-Bo},
  journal={Applied and Computational Harmonic Analysis},
  volume={37},
  number={1},
  pages={89--105},
  year={2014},
  publisher={Elsevier}
}

@inproceedings{behrmann2019invertible,
  title={Invertible residual networks},
  author={Behrmann, Jens and Grathwohl, Will and Chen, Ricky TQ and Duvenaud, David and Jacobsen, J{\"o}rn-Henrik},
  booktitle={International conference on machine learning},
  pages={573--582},
  year={2019},
  organization={PMLR}
}

@inproceedings{agustsson2017ntire,
  title={Ntire 2017 challenge on single image super-resolution: Dataset and study},
  author={Agustsson, Eirikur and Timofte, Radu},
  booktitle={Proceedings of the IEEE conference on computer vision and pattern recognition workshops},
  pages={126--135},
  year={2017}
}

@article{loshchilov2017decoupled,
  title={Decoupled weight decay regularization},
  author={Loshchilov, I},
  journal={arXiv preprint arXiv:1711.05101},
  year={2017}
}

@inproceedings{bevilacqua2012low,
  title={Low-complexity single-image super-resolution based on nonnegative neighbor embedding},
  author={Bevilacqua, Marco and Roumy, Aline and Guillemot, Christine and Morel, Marie-Line Alberi},
  booktitle={British machine vision conference (BMVC)},
  year={2012}
}

@inproceedings{zeyde2012single,
  title={On single image scale-up using sparse-representations},
  author={Zeyde, Roman and Elad, Michael and Protter, Matan},
  booktitle={Curves and Surfaces: 7th International Conference, Avignon, France, June 24-30, 2010, Revised Selected Papers 7},
  pages={711--730},
  year={2012},
  organization={Springer}
}

@inproceedings{martin2001database,
  title={A database of human segmented natural images and its application to evaluating segmentation algorithms and measuring ecological statistics},
  author={Martin, David and Fowlkes, Charless and Tal, Doron and Malik, Jitendra},
  booktitle={Proceedings eighth IEEE international conference on computer vision. ICCV 2001},
  volume={2},
  pages={416--423},
  year={2001},
  organization={IEEE}
}

@inproceedings{huang2015single,
  title={Single image super-resolution from transformed self-exemplars},
  author={Huang, Jia-Bin and Singh, Abhishek and Ahuja, Narendra},
  booktitle={Proceedings of the IEEE conference on computer vision and pattern recognition},
  pages={5197--5206},
  year={2015}
}

@inproceedings{ahn2018fast,
  title={Fast, accurate, and lightweight super-resolution with cascading residual network},
  author={Ahn, Namhyuk and Kang, Byungkon and Sohn, Kyung-Ah},
  booktitle={Proceedings of the European conference on computer vision (ECCV)},
  pages={252--268},
  year={2018}
}

@article{zhang2021residual,
  title={Residual Dense Network for Image Restoration},
  author={Zhang, Yulun and Tian, Yapeng and Kong, Yu and Zhong, Bineng and Fu, Yun},
  journal={IEEE transactions on pattern analysis and machine intelligence},
  volume={43},
  number={7},
  pages={2480--2495},
  year={2021}
}

@inproceedings{zhang2018image,
  title={Image super-resolution using very deep residual channel attention networks},
  author={Zhang, Yulun and Li, Kunpeng and Li, Kai and Wang, Lichen and Zhong, Bineng and Fu, Yun},
  booktitle={Proceedings of the European conference on computer vision (ECCV)},
  pages={286--301},
  year={2018}
}

@inproceedings{dai2019second,
  title={Second-order attention network for single image super-resolution},
  author={Dai, Tao and Cai, Jianrui and Zhang, Yongbing and Xia, Shu-Tao and Zhang, Lei},
  booktitle={Proceedings of the IEEE/CVF conference on computer vision and pattern recognition},
  pages={11065--11074},
  year={2019}
}

@inproceedings{chen2023activating,
  title={Activating more pixels in image super-resolution transformer},
  author={Chen, Xiangyu and Wang, Xintao and Zhou, Jiantao and Qiao, Yu and Dong, Chao},
  booktitle={Proceedings of the IEEE/CVF conference on computer vision and pattern recognition},
  pages={22367--22377},
  year={2023}
}

@inproceedings{shin2017jpeg,
  title={Jpeg-resistant adversarial images},
  author={Shin, Richard and Song, Dawn},
  booktitle={NIPS 2017 workshop on machine learning and computer security},
  volume={1},
  pages={8},
  year={2017}
}

@inproceedings{jiang2021towards,
  title={Towards flexible blind JPEG artifacts removal},
  author={Jiang, Jiaxi and Zhang, Kai and Timofte, Radu},
  booktitle={Proceedings of the IEEE/CVF International Conference on Computer Vision},
  pages={4997--5006},
  year={2021}
}

@article{ma2016waterloo,
  title={Waterloo exploration database: New challenges for image quality assessment models},
  author={Ma, Kede and Duanmu, Zhengfang and Wu, Qingbo and Wang, Zhou and Yong, Hongwei and Li, Hongliang and Zhang, Lei},
  journal={IEEE Transactions on Image Processing},
  volume={26},
  number={2},
  pages={1004--1016},
  year={2016},
  publisher={IEEE}
}

@article{zhang2021plug,
  title={Plug-and-play image restoration with deep denoiser prior},
  author={Zhang, Kai and Li, Yawei and Zuo, Wangmeng and Zhang, Lei and Van Gool, Luc and Timofte, Radu},
  journal={IEEE Transactions on Pattern Analysis and Machine Intelligence},
  volume={44},
  number={10},
  pages={6360--6376},
  year={2021},
  publisher={IEEE}
}

@article{zhang2017beyond,
  title={Beyond a gaussian denoiser: Residual learning of deep cnn for image denoising},
  author={Zhang, Kai and Zuo, Wangmeng and Chen, Yunjin and Meng, Deyu and Zhang, Lei},
  journal={IEEE transactions on image processing},
  volume={26},
  number={7},
  pages={3142--3155},
  year={2017},
  publisher={IEEE}
}

@misc{kodak24,
  author = {Franzen, Rich},
  title = {{Kodak Lossless True Color Image Suite}},
  year = {1999},
  url = {http://r0k.us/graphics/kodak/},
  note = {Accessed: 2024-05-20}
}

@article{zhang2011color,
  title={Color demosaicking by local directional interpolation and nonlocal adaptive thresholding},
  author={Zhang, Lei and Wu, Xiaolin and Buades, Antoni and Li, Xin},
  journal={Journal of Electronic imaging},
  volume={20},
  number={2},
  pages={023016--023016},
  year={2011},
  publisher={Society of Photo-Optical Instrumentation Engineers}
}

@inproceedings{zhang2017learning,
  title={Learning deep CNN denoiser prior for image restoration},
  author={Zhang, Kai and Zuo, Wangmeng and Gu, Shuhang and Zhang, Lei},
  booktitle={Proceedings of the IEEE conference on computer vision and pattern recognition},
  pages={3929--3938},
  year={2017}
}

@article{zhang2018ffdnet,
  title={FFDNet: Toward a fast and flexible solution for CNN-based image denoising},
  author={Zhang, Kai and Zuo, Wangmeng and Zhang, Lei},
  journal={IEEE Transactions on Image Processing},
  volume={27},
  number={9},
  pages={4608--4622},
  year={2018},
  publisher={IEEE}
}

@article{peng2019dilated,
  title={Dilated residual networks with symmetric skip connection for image denoising},
  author={Peng, Yali and Zhang, Lu and Liu, Shigang and Wu, Xiaojun and Zhang, Yu and Wang, Xili},
  journal={Neurocomputing},
  volume={345},
  pages={67--76},
  year={2019},
  publisher={Elsevier}
}

@inproceedings{xia2020identifying,
  title={Identifying recurring patterns with deep neural networks for natural image denoising},
  author={Xia, Zhihao and Chakrabarti, Ayan},
  booktitle={Proceedings of the IEEE/CVF winter conference on applications of computer vision},
  pages={2426--2434},
  year={2020}
}

@article{tian2020image,
  title={Image denoising using deep CNN with batch renormalization},
  author={Tian, Chunwei and Xu, Yong and Zuo, Wangmeng},
  journal={Neural Networks},
  volume={121},
  pages={461--473},
  year={2020},
  publisher={Elsevier}
}

@article{zhang2019residual,
  title={Residual non-local attention networks for image restoration},
  author={Zhang, Yulun and Li, Kunpeng and Li, Kai and Zhong, Bineng and Fu, Yun},
  journal={arXiv preprint arXiv:1903.10082},
  year={2019}
}

@article{daubechies2004iterative,
  title={An iterative thresholding algorithm for linear inverse problems with a sparsity constraint},
  author={Daubechies, Ingrid and Defrise, Michel and De Mol, Christine},
  journal={Communications on Pure and Applied Mathematics: A Journal Issued by the Courant Institute of Mathematical Sciences},
  volume={57},
  number={11},
  pages={1413--1457},
  year={2004},
  publisher={Wiley Online Library}
}

@article{rudin1992nonlinear,
  title={Nonlinear total variation based noise removal algorithms},
  author={Rudin, Leonid I and Osher, Stanley and Fatemi, Emad},
  journal={Physica D: nonlinear phenomena},
  volume={60},
  number={1-4},
  pages={259--268},
  year={1992},
  publisher={Elsevier}
}

@inproceedings{shi2016real,
  title={Real-time single image and video super-resolution using an efficient sub-pixel convolutional neural network},
  author={Shi, Wenzhe and Caballero, Jose and Husz{\'a}r, Ferenc and Totz, Johannes and Aitken, Andrew P and Bishop, Rob and Rueckert, Daniel and Wang, Zehan},
  booktitle={Proceedings of the IEEE conference on computer vision and pattern recognition},
  pages={1874--1883},
  year={2016}
}

@inproceedings{timofte2016seven,
  title={Seven ways to improve example-based single image super resolution},
  author={Timofte, Radu and Rothe, Rasmus and Van Gool, Luc},
  booktitle={Proceedings of the IEEE conference on computer vision and pattern recognition},
  pages={1865--1873},
  year={2016}
}

@inproceedings{bao2025tinvblock,
  author    = {Jingwei Bao and Jinhua Hao and Pengcheng Xu and Ming Sun and Chao Zhou and Shuyuan Zhu},
  title     = {Plug-and-Play Tri‑Branch Invertible Block for Image Rescaling},
  booktitle = {Proceedings of the AAAI Conference on Artificial Intelligence},
  year      = {2025},
  volume    = {39},
  pages     = {1826--1834},
  month     = apr,
}

@article{ruderman1993statistics,
  title={Statistics of natural images: Scaling in the woods},
  author={Ruderman, Daniel and Bialek, William},
  journal={Advances in neural information processing systems},
  volume={6},
  year={1993}
}

@article{balle2016end,
  title={End-to-end optimized image compression},
  author={Ball{\'e}, Johannes and Laparra, Valero and Simoncelli, Eero P},
  journal={arXiv preprint arXiv:1611.01704},
  year={2016}
}

@article{balle2018variational,
  title={Variational image compression with a scale hyperprior},
  author={Ball{\'e}, Johannes and Minnen, David and Singh, Saurabh and Hwang, Sung Jin and Johnston, Nick},
  journal={arXiv preprint arXiv:1802.01436},
  year={2018}
}

@article{choi2022scalable,
  title={Scalable image coding for humans and machines},
  author={Choi, Hyomin and Baji{\'c}, Ivan V},
  journal={IEEE Transactions on Image Processing},
  volume={31},
  pages={2739--2754},
  year={2022},
  publisher={IEEE}
}

@inproceedings{shao2020bottlenet++,
  title={Bottlenet++: An end-to-end approach for feature compression in device-edge co-inference systems},
  author={Shao, Jiawei and Zhang, Jun},
  booktitle={2020 IEEE International Conference on Communications Workshops (ICC Workshops)},
  pages={1--6},
  year={2020},
  organization={IEEE}
}

@inproceedings{wang2025timestep,
  title={Timestep-Aware Diffusion Model for Extreme Image Rescaling},
  author={Wang, Ce and Hu, Zhenyu and Sun, Wanjie and Chen, Zhenzhong},
  booktitle={Proceedings of the IEEE/CVF International Conference on Computer Vision},
  pages={15594--15603},
  year={2025}
}

@article{field1987relations,
  title={Relations between the statistics of natural images and the response properties of cortical cells},
  author={Field, David J},
  journal={Journal of the Optical Society of America A},
  volume={4},
  number={12},
  pages={2379--2394},
  year={1987},
  publisher={OSA}
}

@incollection{coifman1995translation,
  title={Translation-invariant de-noising},
  author={Coifman, Ronald R and Donoho, David L},
  booktitle={Wavelets and statistics},
  pages={125--150},
  year={1995},
  publisher={Springer}
}

@article{cai2009linearized,
  title={Linearized Bregman iterations for frame-based image deblurring},
  author={Cai, Jian-Feng and Osher, Stanley and Shen, Zuowei},
  journal={SIAM Journal on Imaging Sciences},
  volume={2},
  number={1},
  pages={226--252},
  year={2009},
  publisher={SIAM}
}

@article{donoho2002noising,
  title={De-noising by soft-thresholding},
  author={Donoho, David L},
  journal={IEEE transactions on information theory},
  volume={41},
  number={3},
  pages={613--627},
  year={2002},
  publisher={IEEE}
}

@article{aharon2006k,
  title={K-SVD: An algorithm for designing overcomplete dictionaries for sparse representation},
  author={Aharon, Michal and Elad, Michael and Bruckstein, Alfred},
  journal={IEEE Transactions on signal processing},
  volume={54},
  number={11},
  pages={4311--4322},
  year={2006},
  publisher={IEEE}
}

@inproceedings{zhang2020approximation,
  title={Approximation capabilities of neural ODEs and invertible residual networks},
  author={Zhang, Han and Gao, Xi and Unterman, Jacob and Arodz, Tom},
  booktitle={International Conference on Machine Learning},
  pages={11086--11095},
  year={2020},
  organization={PMLR}
}

@article{elad2006image,
  title={Image denoising via sparse and redundant representations over learned dictionaries},
  author={Elad, Michael and Aharon, Michal},
  journal={IEEE Transactions on Image processing},
  volume={15},
  number={12},
  pages={3736--3745},
  year={2006},
  publisher={IEEE}
}

@article{goyal2001quantized,
  title={Quantized frame expansions with erasures},
  author={Goyal, Vivek K and Kova{\v{c}}evi{\'c}, Jelena and Kelner, Jonathan A},
  journal={Applied and Computational Harmonic Analysis},
  volume={10},
  number={3},
  pages={203--233},
  year={2001},
  publisher={Elsevier}
}

@inproceedings{huang2017densely,
  title={Densely Connected Convolutional Networks},
  author={Huang, Gao and Liu, Zhuang and Van Der Maaten, Laurens and Weinberger, Kilian Q.},
  booktitle={Proceedings of the IEEE Conference on Computer Vision and Pattern Recognition},
  pages={4700--4708},
  year={2017}
}

@inproceedings{wang2018esrgan,
  title={ESRGAN: Enhanced Super-Resolution Generative Adversarial Networks},
  author={Wang, Xintao and Yu, Ke and Wu, Shixiang and Gu, Jinjin and Liu, Yihao and Dong, Chao and Loy, Chen Change and Qiao, Yu and Tang, Xiaoou},
  booktitle={Proceedings of the European Conference on Computer Vision Workshops},
  year={2018}
}
\end{document}